\documentclass[preprint,12pt,a4paper]{elsarticle}

\usepackage{amsmath,amssymb,amsthm}
\usepackage{booktabs,graphicx,array,enumitem,multirow,rotating,longtable,adjustbox,float}
\usepackage[ruled,vlined,linesnumbered]{algorithm2e}
\usepackage[table]{xcolor}
\usepackage{microtype}
\usepackage{pdflscape}
\usepackage{pifont}
\usepackage{dsfont}
\usepackage{bbm}
\usepackage{url}
\usepackage[hidelinks]{hyperref}
\biboptions{numbers,sort&compress}

\usepackage{titlesec}
\renewcommand{\thesection}{\arabic{section}}
\renewcommand{\thesubsection}{\thesection.\arabic{subsection}}
\renewcommand{\thesubsubsection}{\thesubsection.\arabic{subsubsection}}
\titleformat{\section}[hang]{\normalfont\large\bfseries}{\thesection.}{0.55em}{}
\titleformat{\subsection}[hang]{\normalfont\normalsize\bfseries}{\thesubsection.}{0.55em}{}
\titleformat{\subsubsection}[hang]{\normalfont\normalsize\itshape}{\thesubsubsection.}{0.55em}{}
\titlespacing*{\section}{0pt}{2.0ex plus 0.8ex minus 0.2ex}{0.9ex plus 0.2ex}
\titlespacing*{\subsection}{0pt}{1.6ex plus 0.6ex minus 0.2ex}{0.7ex plus 0.2ex}
\titlespacing*{\subsubsection}{0pt}{1.2ex plus 0.5ex minus 0.2ex}{0.5ex plus 0.2ex}

\makeatletter
\renewcommand{\figurename}{Figure}
\renewcommand{\tablename}{Table}
\renewcommand{\thetable}{\arabic{table}}
\def\fnum@figure{\figurename\nobreakspace\thefigure}
\def\fnum@table{\tablename\nobreakspace\thetable}
\makeatother

\newcommand{\statpara}[1]{\vspace{0.5mm}\noindent #1\par\vspace{0.5mm}}

\graphicspath{{Figures/}{./}}

\newcommand{\cmark}{\ding{51}}
\newcommand{\xmark}{\ding{55}}
\newcommand{\pmark}{--}

\newcommand{\R}{\mathbb{R}}
\newcommand{\N}{\mathbb{N}}

\newcommand{\Pset}{\mathcal{P}}

\newcommand{\Acal}{\mathcal{A}}
\newcommand{\argmin}{\operatorname{argmin}}

\DeclareMathOperator{\clip}{clip}

\newcommand{\abs}[1]{\lvert #1 \rvert}
\newcommand{\given}{\mid}

\newcommand{\Xcal}{\mathcal{X}}

\newtheorem{assumption}{Assumption}

\newtheorem{proposition}{Proposition}
\newtheorem{theorem}{Theorem}
\newtheorem{remark}{Remark}

\renewcommand{\arraystretch}{1.08}
\begin{document}

\begin{frontmatter}

\title{Coronavirus Optimization Algorithm: A Success-History Adaptive Evolutionary Framework with Archive-Assisted Search and Stagnation Recovery for Global Optimization}

\author[inst1]{Hari Mohan Pandey}
\ead{profharimohanpandey@gmail.com}
\affiliation[inst1]{Department of Computing & Engineering, Bournemouth University, UK}

\begin{abstract}
Population-based metaheuristics remain widely used for black-box global optimization, yet the rapid growth of metaphor-based optimizers has raised concerns about weak mathematical grounding, limited operator transparency, and insufficient validation against strong baselines. This paper proposes the Coronavirus Optimization Algorithm (COA), a severe acute respiratory syndrome coronavirus 2 (SARS-CoV-2)-inspired adaptive evolutionary optimizer for box-constrained continuous global optimization. COA does not model epidemiological infection or transmission dynamics; instead, selected coronavirus-related mechanisms are translated into explicit computational operators through a formal biological-to-optimization mapping. Spike--receptor binding is represented through elite-guided attraction, viral replication through trial-vector generation, antigenic drift through adaptive variation, immune evasion through stagnation recovery, and viral-load dynamics through population scheduling. The executable algorithm integrates opposition-based initialization, Differential Evolution (DE)/current-to-pbest/1 mutation, binomial crossover, an external archive, success-history adaptation of the mutation scale factor and crossover rate, population-size reduction, and opposition-based partial restart. The proposed method is evaluated on 29 Congress on Evolutionary Computation (CEC) 2017 benchmark functions at dimensions 10, 30, and 50 against 15 competitive optimizers, including established evolutionary, swarm-intelligence, and recent metaphor-based algorithms. The consolidated results show that COA obtains the best overall average Friedman rank at all tested dimensions, with ranks of 2.79, 2.86, and 2.53 for 10D, 30D, and 50D problems, respectively. Its strongest performance is observed on composition functions, where it achieves the best mean objective value across all composition benchmarks at 30D and 50D. Category-wise analysis also reveals that COA is not uniformly dominant, with stronger competing behaviour from Grey Wolf Optimizer (GWO) on hybrid functions. These findings demonstrate that COA is a competitive, compact, and transparent adaptive evolutionary optimizer, particularly effective for complex composition landscapes, while also identifying its limitations and future validation requirements for higher-dimensional search.
\end{abstract}


\end{frontmatter}

\section{Introduction}
\label{sec:intro}

Continuous global optimization is a core problem in evolutionary computation, computational intelligence, engineering design, machine learning, control, and scientific computing. This paper considers the box-constrained minimization problem
\begin{equation}
    x^* \in \argmin_{x\in\Omega} f(x), 
    \qquad 
    \Omega = \prod_{j=1}^{D}[lb_j,ub_j] \subset \R^{D},
    \label{eq:global_opt}
\end{equation}
where $f:\Omega\to\R$ is a black-box objective function and $D$ is the problem dimension. In many real applications, $f$ may be nonlinear, multimodal, nonconvex, noisy, discontinuous, or expensive to evaluate. These properties make derivative-based and deterministic global methods difficult to apply, particularly when the search space becomes large or irregular.

Population-based metaheuristics address this setting by updating a set of candidate solutions using only objective-function evaluations. Differential Evolution (DE), Particle Swarm Optimization (PSO), Covariance Matrix Adaptation Evolution Strategy (CMA-ES), and adaptive DE variants such as Joint Adaptive Differential Evolution (JADE), Success-History based Adaptive Differential Evolution (SHADE), and Linear Success-History based Adaptive Differential Evolution with modified CMA-ES (LSHADE-SPACMA) have established strong foundations for derivative-free numerical optimization~\cite{DE,PSO,CMAES,JADE,SHADE,LSHADE,LSHADE-SPACMA}. More recent swarm and metaphor-based methods, including Grey Wolf Optimizer (GWO), Harris Hawks Optimization (HHO), Whale Optimization Algorithm (WOA), Salp Swarm Algorithm (SSA), Aquila Optimizer (AO), Runge--Kutta optimizer (RUN), Rime Optimizer (RIME), Dwarf Mongoose Optimization algorithm (DMO), and Crested Porcupine Optimizer (CPO), further expand the design space of population movement, attraction, exploration, and exploitation mechanisms~\cite{GWO,HHO,WOA,SSA,AO,RUN,RIME,DMO,CPO}. However, the performance of a new optimizer should be justified by its computational operators, ablation evidence, statistical behaviour, and reproducibility, rather than by metaphor alone.

This paper proposes the \textit{Coronavirus Optimization Algorithm} (COA), a severe acute respiratory syndrome coronavirus 2 (SARS-CoV-2)-inspired success-history adaptive evolutionary optimizer for problem~\eqref{eq:global_opt}. COA does not model infection, transmission, immunity, or public-health dynamics. Instead, selected SARS-CoV-2-related concepts are used as an organizing abstraction for optimization behaviour. Spike--receptor binding is translated into elite-guided attraction, viral replication into trial-vector generation, antigenic drift into adaptive parameter variation, immune evasion into opposition-based stagnation recovery, and viral-load dynamics into population-size reduction. The resulting executable algorithm combines opposition-based initialization, DE/current-to-pbest/1 mutation, binomial crossover, an external archive, success-history adaptation of the mutation factor $F$ and crossover rate $CR$, scheduled population reduction, and opposition-based partial restart.

The experimental study evaluates COA on 29 Congress on Evolutionary Computation (CEC) 2017 benchmark functions at $D=10$, $D=30$, and $D=50$ using 30 independent runs per configuration. The results show that COA achieves the lowest average Friedman rank across all tested dimensions, with $\bar{R}_{\text{COA}}(10)=2.79$, $\bar{R}_{\text{COA}}(30)=2.86$, and $\bar{R}_{\text{COA}}(50)=2.53$. The strongest and most consistent gains occur on composition functions, where COA benefits from the interaction of elite guidance, archive-assisted diversity, adaptive parameter control, and restart behaviour. The results are more mixed on hybrid functions, where GWO remains a strong competitor. Therefore, the paper does not claim universal superiority; instead, it presents COA as a competitive adaptive evolutionary optimizer with clear strengths, measurable limitations, and reproducible empirical evidence.

The main contributions are summarized as follows:
\begin{enumerate}[leftmargin=*]
    \item COA is introduced as a mathematically specified adaptive evolutionary optimizer for box-constrained global optimization.
    \item A SARS-CoV-2-inspired mapping is defined between biological concepts and concrete search operators, including elite-guided attraction, replication, adaptive variation, stagnation recovery, and population scheduling.
    \item The method integrates established evolutionary mechanisms, including current-to-pbest mutation, archive-assisted diversity, success-history adaptation, opposition-based restart, and population-size reduction.
    \item Formal properties and detailed proofs are included as main-paper sections, including feasibility preservation, best-so-far monotonicity, population-size validity, finite termination, and evaluation complexity.
    \item A consolidated evaluation against 15 baseline optimizers is reported on 29 CEC 2017 functions at $D\in\{10,30,50\}$ with 30 independent runs, using ranks, per-function results, convergence evidence, win counts, heatmaps, and category-wise analysis.
\end{enumerate}

The remainder of the paper is organized as follows. Section~2 reviews related work; Section~3 presents the SARS-CoV-2-inspired mapping; Section~4 presents the proposed COA algorithm together with implementation details and complexity analysis; Section~5 provides the mathematical foundation, formal properties, and proofs; Section~6 reports the experimental setup, results, and analysis; and Section~7 concludes the paper. 

\section{Related Work}
\label{sec:related}

Continuous black-box optimization has been extensively studied through evolutionary, swarm-intelligence, covariance-based, and recent metaphor-inspired methods. Classical DE generates new candidate solutions from scaled vector differences and remains a strong baseline because of its simplicity and robustness~\cite{DE}. Later adaptive DE variants, including JADE, SHADE, and LSHADE-SPACMA, improved this framework through current-to-pbest mutation, external archives, success-history parameter memories, population-size reduction, and covariance-related hybridization~\cite{JADE,SHADE,LSHADE,LSHADE-SPACMA}. COA follows this adaptive evolutionary line: its novelty lies in combining elite-guided mutation, archive-assisted diversity, success-history adaptation, opposition-based restart, and population scheduling within one compact search procedure.

Swarm and metaphor-based optimizers provide complementary search behaviours. PSO uses social learning, while GWO, HHO, WOA, and SSA model hierarchy, pursuit, encircling, and chain-based movement~\cite{PSO,GWO,HHO,WOA,SSA}. More recent methods such as AO, RUN, RIME, DMO, and CPO introduce additional physical, numerical, or animal-inspired update rules~\cite{AO,RUN,RIME,DMO,CPO}. These methods are useful comparators because they represent different exploration--exploitation biases. The present results show that AO is the closest overall competitor to COA, while GWO is particularly strong on hybrid functions.

Covariance-based search, represented by CMA-ES, adapts the shape of the sampling distribution and is often effective on ill-conditioned landscapes~\cite{CMAES}. COA does not perform explicit covariance learning; instead, it relies on scalar adaptation of $F$ and $CR$, archive diversity, and restart mechanisms. This design keeps the algorithm simple and computationally compact, but it also explains why strongly coupled hybrid functions remain challenging. The benchmarking literature recommends that such claims be supported by function-level, category-wise, and statistical evidence rather than aggregate rank alone~\cite{DERRAC,GARCIA}.

Table~\ref{tab:positioning} summarizes the operator-level positioning of COA. The method integrates adaptive parameter control, archive usage, restart, and population scheduling, but not covariance learning. This positioning clarifies that COA should be understood as an adaptive evolutionary optimizer with a coronavirus-inspired explanatory abstraction, rather than as a metaphor-only method.

\begin{table}[!htbp]
\centering
\caption{Operator-level positioning of COA.}
\label{tab:positioning}
\footnotesize
\renewcommand{\arraystretch}{1.12}
\setlength{\tabcolsep}{3.0pt}
\begin{adjustbox}{max width=\columnwidth}
\begin{tabular}{p{2.9cm}ccccc}
\toprule
\textbf{Optimizer family} 
& \textbf{Adapt.} 
& \textbf{Archive} 
& \textbf{Restart} 
& \textbf{Pop. sched.} 
& \textbf{Cov. learn.} \\
\midrule
Classical DE~\cite{DE} 
& \xmark & \xmark & \xmark & \xmark & \xmark \\
Adaptive DE~\cite{JADE,SHADE,LSHADE-SPACMA} 
& \cmark & \cmark & \xmark & \cmark & \pmark \\
Swarm intelligence~\cite{PSO,GWO,HHO,WOA,SSA} 
& \pmark & \xmark & \xmark & \xmark & \xmark \\
Covariance-based search~\cite{CMAES} 
& \cmark & \xmark & \xmark & \xmark & \cmark \\
Recent metaphor-based methods~\cite{AO,RUN,RIME,DMO,CPO} 
& \pmark & \xmark & \pmark & \pmark & \xmark \\
\textbf{COA} 
& \cmark & \cmark & \cmark & \cmark & \xmark \\
\bottomrule
\end{tabular}
\end{adjustbox}

\vspace{1mm}
\begin{minipage}{\columnwidth}
\footnotesize
\textit{Note:} Adapt. = adaptive parameter control; Pop. sched. = population-size scheduling; Cov. learn. = covariance or landscape-geometry learning; \cmark~= supported; \xmark~= not supported; \pmark~= partial or indirect support.
\end{minipage}
\end{table}

\statpara{Table~\ref{tab:positioning} positions COA against five optimizer families using five operator criteria. COA is the only listed method that simultaneously includes adaptive control, archive-assisted diversity, restart, and population scheduling, while it deliberately avoids covariance learning. This supports the classification of COA as an adaptive evolutionary optimizer rather than a purely metaphor-driven method.}

\section{COA Inspired by SARS-CoV-2}
\label{sec:covid_inspiration}

COA is inspired by selected behavioural stages associated with SARS-CoV-2, but it is not a virological, clinical, or epidemiological model. The inspiration is used as a structured computational metaphor for global optimization. A candidate solution is treated as a possible viral variant in the search space, the objective function is the selection pressure, and the optimization process imitates a sequence of entry, replication, variation, evasion, and population reduction. The purpose of this section is to make the inspiration explicit and to connect each SARS-CoV-2-inspired concept to the mathematical operator used in COA.

Let the feasible search space be
\begin{equation}
    \Omega = \prod_{j=1}^{D}[lb_j,ub_j],
\end{equation}
and let each candidate solution be represented by
\begin{equation}
    x_i=(x_{i,1},x_{i,2},\ldots,x_{i,D})\in\Omega.
\end{equation}
For a minimization problem, a lower objective value corresponds to a more successful candidate. Hence, the evolutionary pressure in COA is expressed through the comparison of objective values rather than through a biological infection process.

\begin{table}[!htbp]
\centering
\caption{SARS-CoV-2-inspired concepts and their roles in COA.}
\label{tab:mapping}
\footnotesize
\renewcommand{\arraystretch}{1.12}
\setlength{\tabcolsep}{3.0pt}
\begin{adjustbox}{max width=\columnwidth}
\begin{tabular}{p{2.35cm}p{2.25cm}p{3.05cm}}
\toprule
\textbf{Concept} & \textbf{Search role} & \textbf{COA implementation} \\
\midrule
Spike--receptor binding & Attraction toward promising regions & Elite-guided current-to-pbest mutation~\cite{JADE} \\
Viral replication & Candidate generation & Mutation with binomial crossover~\cite{DE} \\
Antigenic drift & Controlled variation & Success-history adaptation of $F$ and $CR$~\cite{SHADE} \\
Immune evasion & Stagnation escape & Opposition-based partial restart~\cite{OPPOSITION,RAHNAMAYAN} \\
Viral-load dynamics & Exploration--exploitation transition & Population-size reduction~\cite{LSHADE} \\
\bottomrule
\end{tabular}
\end{adjustbox}
\end{table}

\statpara{Table~\ref{tab:mapping} provides a one-to-one translation from SARS-CoV-2-inspired concepts to executable optimization operators. The mapping clarifies that each metaphorical component has a concrete algorithmic role: attraction, trial generation, parameter adaptation, stagnation recovery, or population scheduling.}

Table~\ref{tab:mapping} summarizes the complete metaphor-to-mathematics translation. The key point is that the proposed optimizer does not rely on metaphor alone. Each biological idea is converted into an explicit operator that either changes the candidate solution, controls the population, adapts a parameter, or recovers diversity.

\subsection{Biological mechanisms of SARS-CoV-2 and transformation into metaheuristic operators}
\label{subsec:biology_to_metaheuristic}

Let $\mathcal{B} = \{b_1,b_2,b_3,b_4,b_5\}$ denote the set of SARS-CoV-2 biological mechanisms and $\mathcal{O} = \{o_1,o_2,o_3,o_4,o_5\}$ the corresponding optimization operators. Define the transformation mapping $\varphi: \mathcal{B} \to \mathcal{O}$ as a bijection that translates each biological mechanism into an explicit computational operator. Table~\ref{tab:mapping} gives the complete mapping. SARS-CoV-2, the betacoronavirus responsible for the coronavirus disease 2019 (COVID-19) pandemic, exhibits a structured infection cycle with five main phases: (i) host-cell attachment via spike--angiotensin-converting enzyme 2 (ACE2) receptor binding ($b_1$), (ii) genome release and replication via RNA-dependent RNA polymerase ($b_2$), (iii) mutation-driven antigenic variation ($b_3$), (iv) immune evasion through epitope change ($b_4$), and (v) viral load dynamics over the infection timeline ($b_5$). Each $b_i$ maps to $o_i = \varphi(b_i)$, where $\varphi$ is defined as follows.

\noindent\textbf{1. Spike--receptor binding $\mapsto$ elite-guided attraction ($\varphi(b_1) = o_1$).}
The spike protein's receptor-binding domain recognizes and attaches to the angiotensin-converting enzyme 2 (ACE2) receptor on the host cell surface~\cite{SARSCoV2Structure}. In optimization, this selectivity maps to a mutation operator that attracts $x_i$ toward an elite solution $x_{pbest}$. Formally,
\begin{equation}
    o_1(x_i \given \Pset_t, \Acal_t) = x_i + F_i(x_{pbest} - x_i) + F_i(x_{r_1} - \hat{x}_{r_2}),
    \label{eq:phi_spike}
\end{equation}
where $x_{pbest}$ is sampled uniformly from the top $p$-fraction of $\Pset_t$, $x_{r_1}\sim U(\Pset_t)$, $\hat{x}_{r_2}\sim U(\Pset_t\cup\Acal_t)$, and $F_i$ is the mutation factor.

\noindent\textbf{2. Replication $\mapsto$ trial-vector generation ($\varphi(b_2) = o_2$).}
The replication--transcription complex produces new genomic RNA~\cite{SARSCoV2Replication}. In COA, replication corresponds to generating a trial vector $u_i$ via binomial crossover of $v_i = o_1(x_i)$ with parent $x_i$:

\begin{equation}
\begin{aligned}
o_2(v_i,x_i) &= u_i,\\[-1mm]
u_{i,j} &=
\begin{cases}
v_{i,j}, & r_j \leq CR_i \ \lor\ j=j_{\mathrm{rand}},\\
x_{i,j}, & \text{otherwise}.
\end{cases}
\end{aligned}
\label{eq:phi_replication}
\end{equation}

with selection pressure $S(u_i, x_i) = \argmin\{f(\Pi_\Omega(u_i)), f(x_i)\}$ enforcing greedy retention.

\noindent\textbf{3. Antigenic drift $\mapsto$ adaptive parameter control ($\varphi(b_3) = o_3$).}
RNA-dependent RNA polymerase lacks proofreading, driving antigenic drift~\cite{SARSCoV2Mutation}. For optimization, this maps to success-history adaptation:

\begin{equation}
\begin{aligned}
o_3:\quad
M_{F,k} &\leftarrow 
\frac{\sum_{F_s\in S_F} F_s^2}
     {\sum_{F_s\in S_F} F_s},\\[-1mm]
M_{CR,k} &\leftarrow 
\frac{1}{|S_{CR}|}\sum_{CR_s\in S_{CR}} CR_s .
\end{aligned}
\label{eq:phi_drift}
\end{equation}

where $S_F, S_{CR}$ are the sets of $F_i, CR_i$ values that produced successful offspring in the current generation.

\noindent\textbf{4. Immune evasion $\mapsto$ stagnation recovery ($\varphi(b_4) = o_4$).}
SARS-CoV-2 evades immune pressure through epitope variation~\cite{SARSCoV2Immune}. In COA, stagnation recovery applies opposition-based mapping to the weakest $k$ individuals:
\begin{equation}
    o_4(x_i) = x'_i,\quad x'_{i,j} = \begin{cases}
        lb_j + ub_j - x_{i,j}, & \text{if } \rho_j < 0.5,\\
        x_{i,j}, & \text{otherwise},
    \end{cases}
    \label{eq:phi_evasion}
\end{equation}
for $\rho_j\sim U(0,1)$, applied only when $\Delta f_{\text{best}} < \epsilon_{\text{stag}}$ over $\tau_{\text{stag}}$ evaluations.

\noindent\textbf{5. Viral-load dynamics $\mapsto$ population scheduling ($\varphi(b_5) = o_5$).}
Viral load follows a rise--peak--decline trajectory~\cite{SARSCoV2ViralLoad}. This maps to nonlinear population-size reduction:

\begin{equation}
\begin{aligned}
o_5(NP_0,NP_{\min},t)
&=\Bigg\lfloor NP_0-(NP_0-NP_{\min}) \\
&\quad \times
\left(\frac{FES_t}{MAX\_FES}\right)^{0.7}
\Bigg\rfloor .
\end{aligned}
\label{eq:phi_viralload}
\end{equation}

which preserves early diversity and concentrates later evaluations.

The transformation from biology to algorithm is not literal; SARS-CoV-2 infection dynamics are far more complex than any computational abstraction. However, the biological narrative provides an intuitive organizational framework for five interconnected optimization mechanisms that together form COA. The following subsections describe each mechanism in mathematical detail.

\subsection{Viral-load dynamics: population-size reduction}
In COVID-19 SARS-CoV-2, viral load changes across infection stages. COA abstracts this idea as a change in the number of active candidate solutions over the search process. A larger population is retained during the early stage to improve global exploration, while a smaller population is used later to concentrate the function-evaluation budget around promising regions.

\begin{figure}[!htbp]
\centering
\includegraphics[width=\columnwidth]{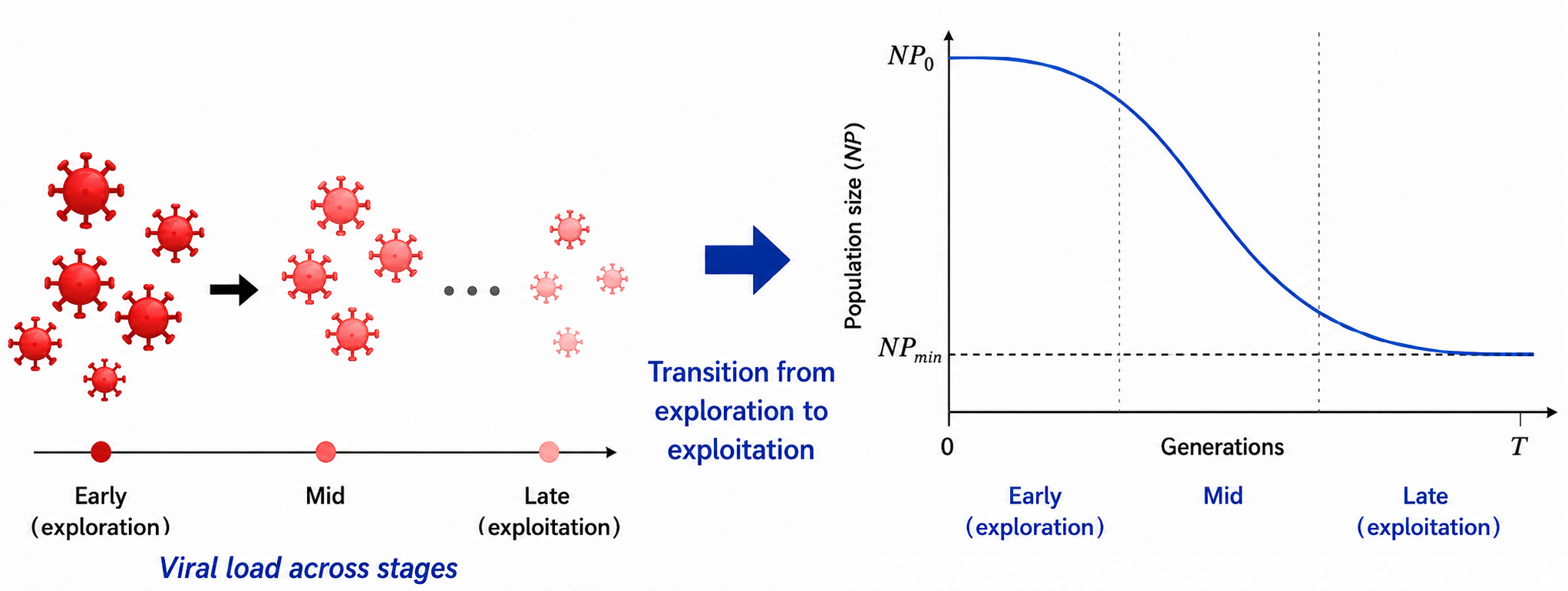}
\caption{Viral-load dynamics as population-size reduction in COA. The left part illustrates the conceptual decrease in viral load across early, middle, and late stages. The right part gives the optimization role: the active population size is gradually reduced from $NP_0$ to $NP_{min}$ as the number of function evaluations increases. This reduction controls the shift from broad exploration to focused exploitation.}
\label{fig:viral_load}
\end{figure}

\statpara{Figure~\ref{fig:viral_load} has no empirical sample statistics; its purpose is to show the mathematical scheduling idea behind COA. Population size decreases from $NP_0$ to $NP_{min}$ as the evaluation counter increases, which turns broad early exploration into more focused late exploitation.}

The mathematical representation of this mechanism is
\begin{equation}
    NP(t)=\left\lfloor NP_0-(NP_0-NP_{min})\left(\frac{FES_t}{MAX\_FES}\right)^{0.7}\right\rfloor,
    \label{eq:covid_population_reduction}
\end{equation}
where $NP(t)$ is the active population size at generation or evaluation stage $t$, $NP_0$ is the initial population size, $NP_{min}$ is the minimum population size, $FES_t$ is the number of function evaluations already consumed, and $MAX\_FES$ is the evaluation budget. The exponent $0.7$ makes the reduction nonlinear: the algorithm preserves diversity early and becomes increasingly exploitative as the run progresses. In the context of the experimental results, this mechanism is important because COA keeps enough candidates to explore multimodal and composition landscapes at the beginning, while later dedicating more evaluations to refining fewer competitive individuals. The nonlinear reduction schedule follows the same principle used in adaptive population-sizing DE frameworks~\cite{LSHADE}.

Figure~\ref{fig:viral_load} illustrates the biological motivation and optimization role of the population-size reduction strategy used in COA. In the biological analogy, the early stage is represented by a high viral load, where many viral particles coexist and spread widely. This corresponds to the exploration phase of the algorithm, in which a large initial population size $NP_0$ is maintained to sample diverse regions of the search space and reduce the risk of premature convergence. As the process progresses to the middle stage, the viral load gradually decreases, reflecting a controlled reduction in the number of active candidate solutions. This stage balances exploration and exploitation by preserving sufficient diversity while increasingly directing the search toward promising areas. In the late stage, only a small number of viral particles remain, representing the exploitation phase. The population size approaches the minimum value $NP_{min}$, allowing the algorithm to concentrate computational effort on local refinement around high-quality solutions.

The right-hand plot formalizes this transition as a population-size reduction schedule. The active population is initially close to $NP_0$, then decreases progressively with the number of generations or function evaluations, and finally stabilizes near $NP_{min}$. This dynamic reduction supports the main search philosophy of COA: broad global exploration is encouraged at the beginning, while focused exploitation becomes dominant near the end of the optimization process. Therefore, the viral-load metaphor provides an intuitive explanation for how COA adaptively controls search diversity, convergence pressure, and computational resource allocation across different optimization stages.

\subsection{Spike--receptor binding: elite-guided attraction}
SARS-CoV-2 enters host cells through spike--receptor binding. COA abstracts this biological entry mechanism as attraction toward high-quality regions of the search space. A current solution $x_i$ is pulled toward an elite solution $x_{pbest}$ while still receiving a diversity-preserving differential perturbation from another population member and the archive.

\begin{figure}[!htbp]
\centering
\includegraphics[width=\columnwidth]{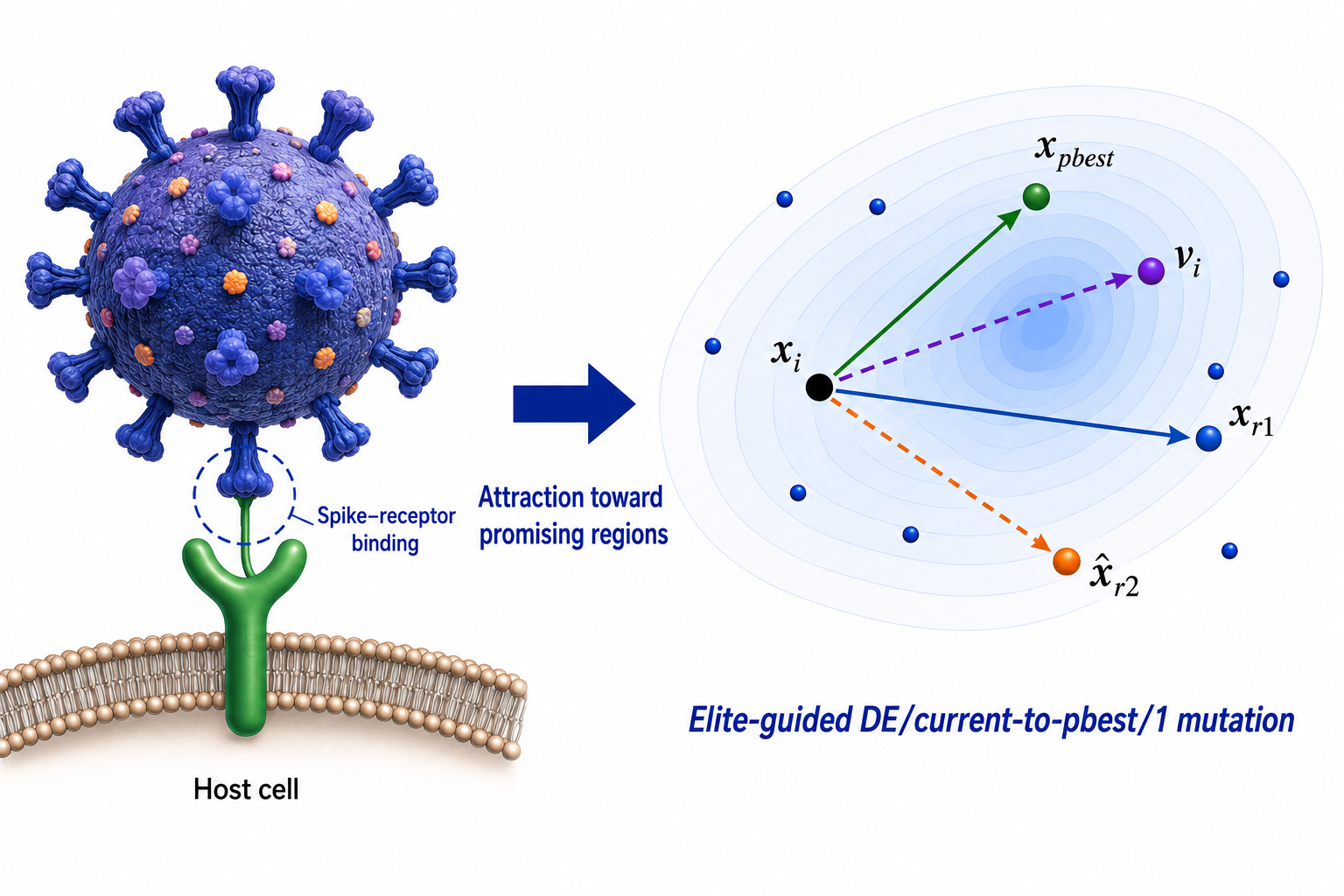}
\caption{Spike--receptor binding as elite-guided attraction in COA. The biological motif of spike--receptor binding is translated into the search movement from the current solution $x_i$ toward an elite solution $x_{pbest}$. The additional difference vector $x_{r_1}-\hat{x}_{r_2}$ prevents purely greedy movement and helps maintain exploratory mobility.}
\label{fig:spike_binding}
\end{figure}

\statpara{Figure~\ref{fig:spike_binding} illustrates the vector statistics used in the mutation step: the current vector, an elite $p$-best vector, and a difference vector sampled from the population/archive. The construction balances exploitation toward high-quality solutions with exploratory displacement from randomly sampled candidates.}

The corresponding mutation operator is represented as
\begin{equation}
    v_i=x_i+F_i(x_{pbest}-x_i)+F_i(x_{r_1}-\hat{x}_{r_2}),
    \label{eq:covid_spike_mutation}
\end{equation}
where $v_i$ is the mutant vector, $F_i$ is the mutation factor, $x_{pbest}$ is sampled from the top-ranked population subset, $x_{r_1}$ is a randomly selected population member, and $\hat{x}_{r_2}$ is sampled from either the population or the external archive. The term $F_i(x_{pbest}-x_i)$ represents receptor-like attraction toward a promising region. The term $F_i(x_{r_1}-\hat{x}_{r_2})$ introduces directional variation using current and historical search information. This combination is central to COA because it avoids two extremes: random wandering without guidance and premature collapse around a single elite candidate. The mutation form inherits the elite-guided search direction from JADE's current-to-pbest design~\cite{JADE}, extended with archive-assisted diversity.

Figure~\ref{fig:spike_binding} explains how COA converts the biological idea of spike--receptor binding into an elite-guided search mechanism. The virus binding to a host receptor represents the attraction of a candidate solution $x_i$ toward a promising elite solution $x_{pbest}$. At the same time, the difference term $x_{r_1}-\hat{x}_{r_2}$ introduces directional variation, avoiding overly greedy convergence and preserving search diversity. This mechanism allows COA to move toward high-quality regions while retaining enough exploratory flexibility.

\subsection{Viral replication: trial-vector generation}
SARS-CoV-2 replication produces new viral copies. In COA, replication is interpreted as the generation of new candidate solutions. The mutant vector created by Eq.~\eqref{eq:covid_spike_mutation} is not automatically accepted. Instead, it is combined with the parent solution through binomial crossover to create a trial vector.

\begin{figure}[!htbp]
\centering
\includegraphics[width=\columnwidth]{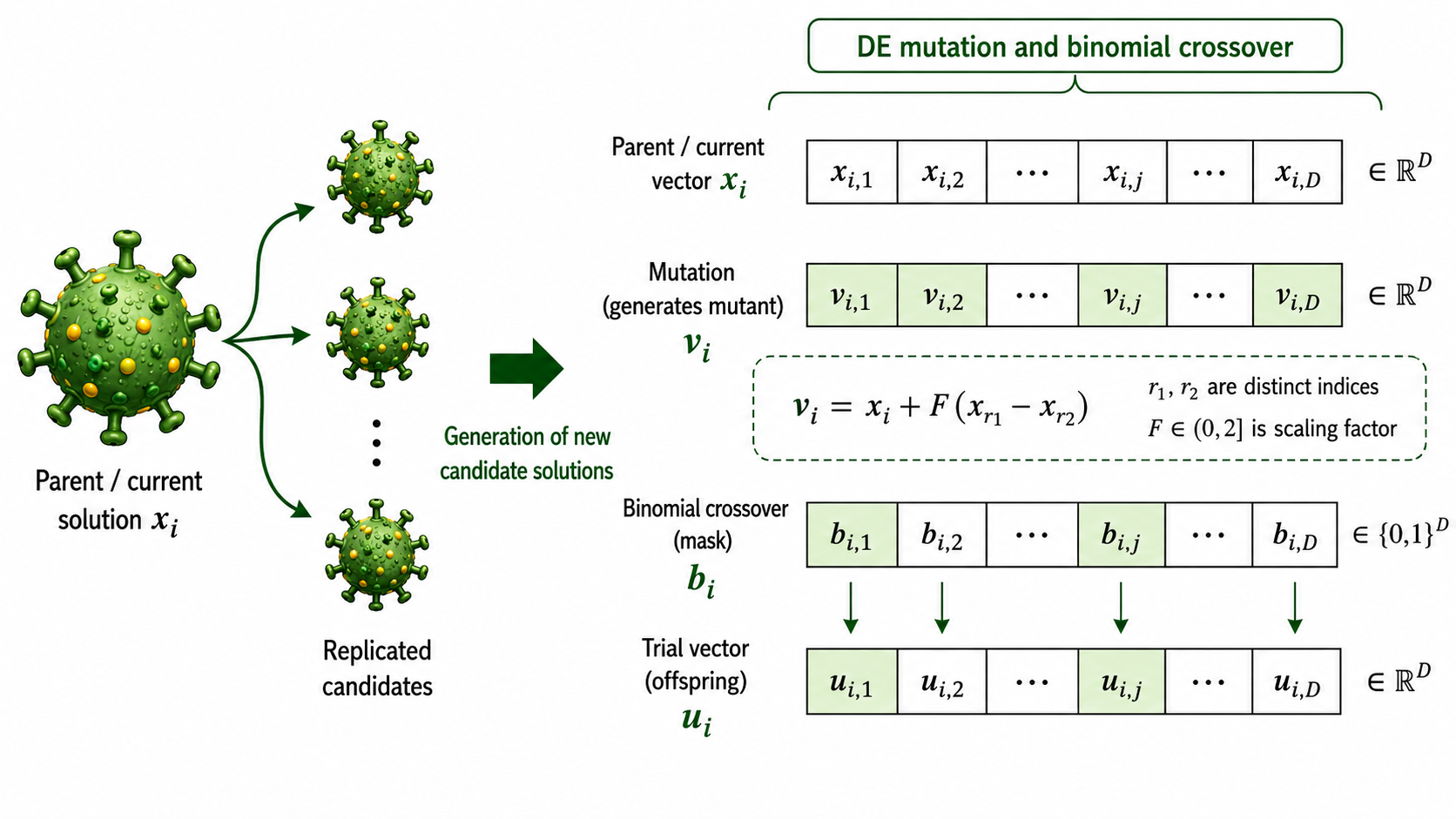}
\caption{Viral replication as mutation and binomial crossover. The figure describes how COA generates a new candidate solution. First, mutation creates a mutant vector $v_i$. Then, binomial crossover forms the trial vector $u_i$ by selecting some dimensions from the mutant vector and the remaining dimensions from the parent vector $x_i$.}
\label{fig:viral_replication}
\end{figure}

\statpara{Figure~\ref{fig:viral_replication} explains how a single parent generates a trial vector. The statistical role of crossover is dimension-wise mixing controlled by $CR$, while the mutation magnitude is controlled by $F$; both parameters are later adapted from successful trials.}

For each dimension $j$, the crossover operation is
\begin{equation}
 u_{i,j}=\begin{cases}
 v_{i,j}, & \text{if } rand_j\leq CR_i \text{ or } j=j_{rand},\\
 x_{i,j}, & \text{otherwise},
 \end{cases}
 \label{eq:covid_replication_crossover}
\end{equation}
where $CR_i$ is the crossover rate and $j_{rand}$ guarantees that at least one dimension is inherited from the mutant vector. The trial vector is then repaired to remain inside the feasible domain and evaluated by the objective function. Selection is greedy:
\begin{equation}
 x_{i,t+1}=\begin{cases}
 \Pi_{\Omega}(u_i), & \text{if } f(\Pi_{\Omega}(u_i))<f(x_{i,t}),\\
 x_{i,t}, & \text{otherwise}.
 \end{cases}
 \label{eq:covid_replication_selection}
\end{equation}
Thus, replication in COA means producing and testing offspring under objective-function selection. This mechanism directly supports the main empirical finding of the paper: COA is strong when repeated generation, evaluation, and selective retention of trial vectors can exploit multiple promising basins, especially on composition functions.

Figure~\ref{fig:viral_replication} shows how COA models viral replication as a candidate-generation process. The parent solution $x_i$ first produces a mutant vector $v_i$ through differential mutation, introducing variation into the search. Binomial crossover then combines selected components of $v_i$ with components of the parent vector to form the trial solution $u_i$. This mechanism enables COA to generate diverse offspring while preserving useful information from the current solution.

\subsection{Antigenic drift: adaptive parameter control}
SARS-CoV-2 variants change over time through mutation and antigenic drift. COA uses this idea to represent adaptive search behaviour. Rather than fixing the mutation factor $F$ and crossover rate $CR$, the algorithm learns useful parameter values from successful trials and stores them in success-history memories.

\begin{figure}[!htbp]
\centering
\includegraphics[width=\columnwidth]{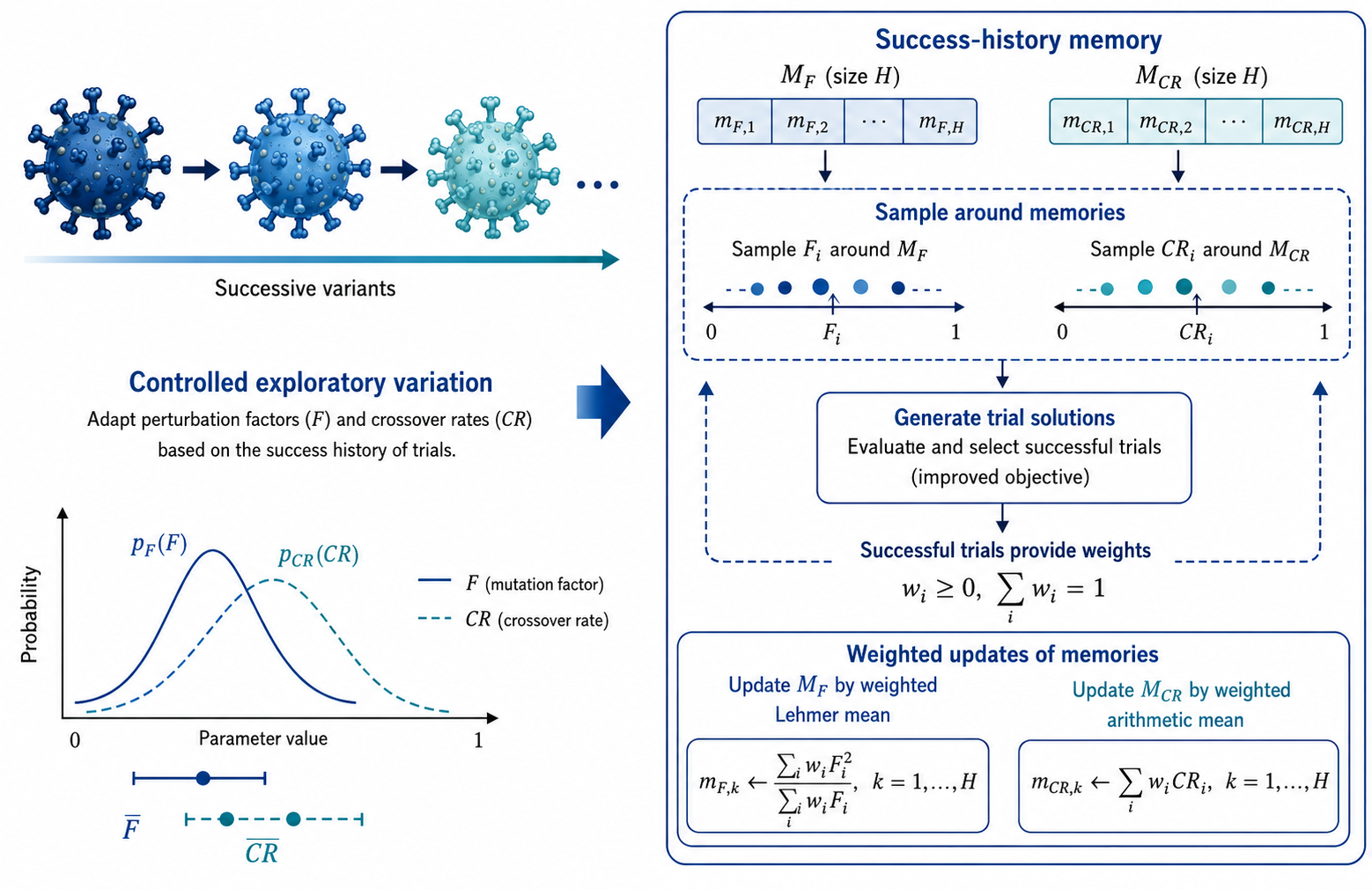}
\caption{Antigenic drift as success-history adaptation of $F$ and $CR$. The figure shows the adaptive loop used in COA. Successful trials contribute values to the memories $M_F$ and $M_{CR}$, and future mutation and crossover parameters are sampled around these memories. This creates controlled exploratory variation rather than static or manually tuned parameter behaviour.}
\label{fig:antigenic_drift}
\end{figure}

\statpara{Figure~\ref{fig:antigenic_drift} summarizes the feedback loop for $F$ and $CR$. Successful parameter samples update the memories $M_F$ and $M_{CR}$, so the distribution used in future generations shifts toward values that recently produced objective improvement.}

Let $M_F$ and $M_{CR}$ be memory arrays of length $H$. For a selected memory index $k$, COA samples
\begin{align}
    F_i &= \operatorname{clip}\left(M_F(k)+0.1\mathcal{N}(0,1),0.1,0.9\right),
    \label{eq:covid_F_sampling}\\
    CR_i &= \operatorname{clip}\left(M_{CR}(k)+0.1\mathcal{N}(0,1),0,1\right).
    \label{eq:covid_CR_sampling}
\end{align}
If $S_F$ and $S_{CR}$ are the successful mutation and crossover values collected in the current generation, the memories are updated as
\begin{align}
    M_F(k) &\leftarrow \frac{\sum_{F_s\in S_F} F_s^2}{\sum_{F_s\in S_F} F_s},
    \label{eq:covid_MF_update}\\
    M_{CR}(k) &\leftarrow \frac{1}{|S_{CR}|}\sum_{CR_s\in S_{CR}} CR_s.
    \label{eq:covid_MCR_update}
\end{align}
The $M_F$ update is a Lehmer-mean style update that gives stronger influence to larger successful mutation factors, while the $M_{CR}$ update captures the average crossover behaviour that produced improvements. The Lehmer-mean weighting for $M_F$ follows the SHADE parameter-adaptation framework~\cite{SHADE}, and the overall adaptation loop reduces sensitivity to manually chosen initial $F$ and $CR$ values compared to static-parameter DE~\cite{BREst}. This parameter adaptation is the mathematical counterpart of antigenic drift: the population changes its exploratory pattern based on successful experience.

Figure~\ref{fig:antigenic_drift} illustrates how COA uses antigenic drift as a metaphor for adaptive parameter control. Successful trial solutions update the memories $M_F$ and $M_{CR}$, which store effective mutation factors and crossover rates from previous generations. New values of $F$ and $CR$ are then sampled around these memories to guide future candidate generation. This feedback loop enables COA to maintain controlled variation, adapt its search behaviour over time, and avoid relying on fixed parameter settings.

\subsection{Immune evasion: opposition-based partial restart}
SARS-CoV-2 can persist by evading immune pressure. In COA, immune evasion is interpreted as recovery from search stagnation. When the algorithm does not improve for a defined interval, it does not restart the whole population. Instead, it identifies weak individuals and partially replaces them using opposition-based mapping.

\begin{figure}[!htbp]
\centering
\includegraphics[width=\columnwidth]{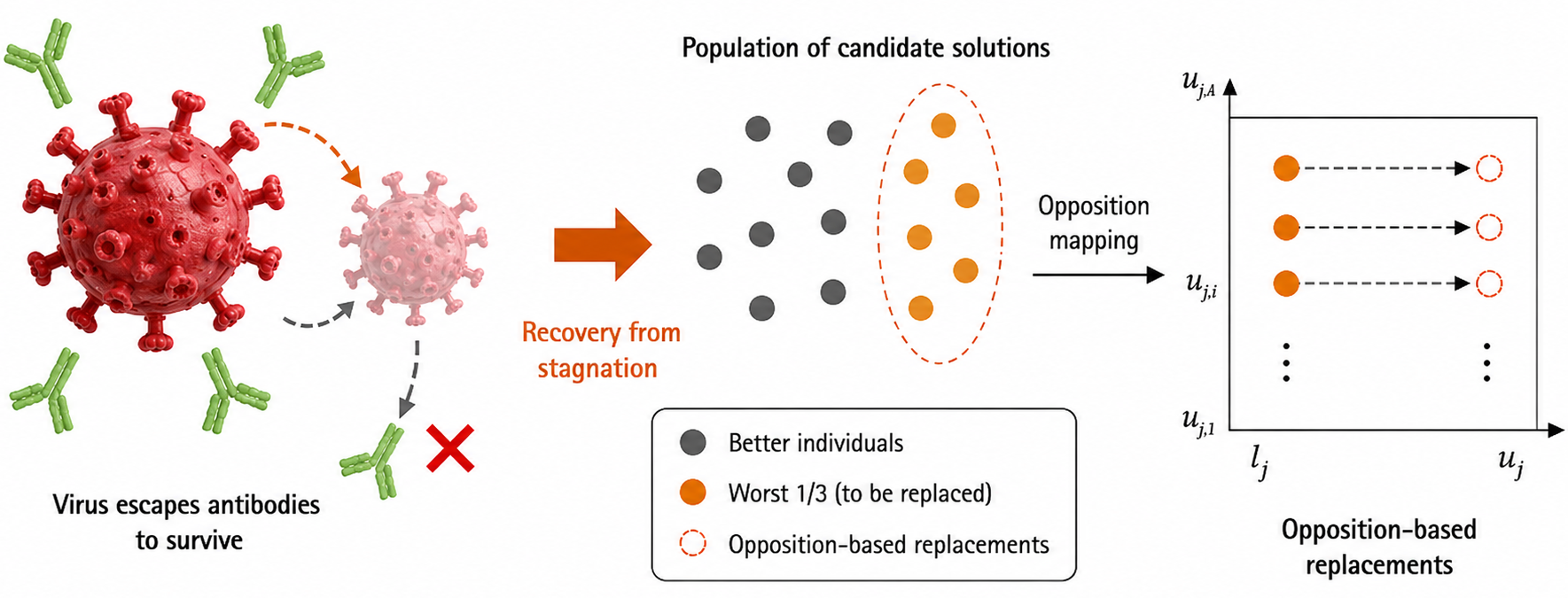}
\caption{Immune evasion as opposition-based partial restart. The figure shows how COA replaces the weakest part of the population when stagnation occurs. Poor candidates are mapped to opposition-based alternatives inside the bounded domain, while stronger candidates and the best-so-far solution are preserved.}
\label{fig:immune_evasion}
\end{figure}

\statpara{Figure~\ref{fig:immune_evasion} describes the stagnation-recovery mechanism. When improvement stalls, COA targets the weakest portion of the population and replaces it with opposition-based candidates, which increases diversity without discarding the best-so-far solution.}

For a weak individual $x_i$, the opposite coordinate is
\begin{equation}
    x'_{i,j}=lb_j+ub_j-x_{i,j}, \qquad j=1,\ldots,D.
    \label{eq:covid_opposition_restart}
\end{equation}
In the implemented partial restart, only a subset of dimensions may be replaced, which can be written as
\begin{equation}
    x^{new}_{i,j}=\begin{cases}
    x'_{i,j}, & \text{if } \rho_j<0.5,\\
    x_{i,j}, & \text{otherwise},
    \end{cases}
    \label{eq:covid_partial_restart}
\end{equation}
where $\rho_j\sim U(0,1)$. This preserves some information from the existing individual while injecting a directionally different candidate into the search. The best-so-far solution and the success-history memories are not discarded. This explains why the restart can improve robustness without destroying the accumulated search knowledge that is needed for later exploitation. The opposition-based restart mechanism draws on the principle that opposite points in the bounded domain can provide useful diversity when the population is stagnating~\cite{OPPOSITION,RAHNAMAYAN}.

Figure~\ref{fig:immune_evasion} illustrates COA’s stagnation-recovery mechanism inspired by immune evasion. When the search begins to stagnate, the weakest fraction of the population is replaced using opposition-based candidates generated within the problem bounds. This partial restart reintroduces diversity without discarding strong solutions or the best-so-far individual, helping the algorithm escape local optima while preserving convergence progress.

Overall, the SARS-CoV-2 inspiration gives COA a coherent design narrative: viral-load dynamics controls population pressure, spike--receptor binding gives elite-guided attraction, replication generates trial vectors, antigenic drift adapts control parameters, and immune evasion restores diversity after stagnation. The experimental results should therefore be interpreted as evidence for this operator combination, not as evidence for biological fidelity.

\section{Proposed COA Algorithm}
\label{sec:method}

\subsection{Optimization problem}
COA solves the bounded black-box minimization problem~\eqref{eq:global_opt}, restated here for completeness:
\begin{equation}
    \min_{x\in\Omega} f(x), \qquad \Omega=\prod_{j=1}^{D}[lb_j,ub_j] \subset \R^{D},
    \label{eq:bounded_problem}
\end{equation}
where $f:\R^{D}\to\R$ is the objective function, $D\in\N$ is the dimension, and $lb_j, ub_j\in\R$ with $lb_j < ub_j$ define the feasible interval for each decision variable $x_j$. The method requires only that $f$ be evaluable at any $x\in\Omega$; it does not require $\nabla f$, convexity, differentiability, or separability.

At generation $t\in\N_0$, the population is the multiset
\begin{equation}
    \Pset_t = \{x_{1,t}, x_{2,t}, \ldots, x_{NP_t,t}\}, \qquad x_{i,t}\in\Omega,\; |\Pset_t| = NP_t,
    \label{eq:population}
\end{equation}
where $NP_t\in\N$ is the current population size satisfying $NP_{\min} \leq NP_t \leq NP_0$. The best-so-far solution up to generation $t$ is
\begin{equation}
    g_t = \argmin_{x\in \bigcup_{\tau=0}^{t}\Pset_\tau} f(x), \qquad f(g_t) = \min_{\tau\leq t}\; \min_{x\in\Pset_\tau} f(x).
    \label{eq:best_so_far}
\end{equation}
COA terminates when the function-evaluation counter (FES) reaches the maximum evaluation budget $MAX\_FES\in\N$.

\subsection{Opposition-based initialization}
Let $\Xcal = \{x_i\}_{i=1}^{NP_0}$ be a set of $NP_0$ candidate solutions sampled independently from $U(\Omega)$. For each $x_i$, COA constructs its opposition point $\tilde{x}_i$ componentwise as
\begin{equation}
    \tilde{x}_{i,j} = lb_j + ub_j - x_{i,j}, \qquad j = 1,\ldots,D.
    \label{eq:opposition}
\end{equation}
Define $\tilde{\Xcal} = \{\tilde{x}_i\}_{i=1}^{NP_0}$. The initial population is
\begin{equation}
    \Pset_0 = \argmin_{NP_0\text{ elements}} \bigcup_{x\in\Xcal\cup\tilde{\Xcal}} f(x),
    \label{eq:init_selection}
\end{equation}
i.e., the $NP_0$ fittest individuals from $\Xcal\cup\tilde{\Xcal}$. This provides better initial domain coverage than random sampling alone~\cite{OPPOSITION,RAHNAMAYAN}. In COA, $NP_0 = 30$ is fixed across all $D\in\{10,30,50\}$.

\subsection{Success-history parameter adaptation}
COA maintains two memory vectors $M_F, M_{CR} \in \R^{H}$ of length $H\in\N$, initialized as $M_F^{(0)}(k) = 0.8$ and $M_{CR}^{(0)}(k) = 0.7$ for all $k=1,\ldots,H$. Let $k_i\sim U(\{1,\ldots,H\})$ be a memory index sampled independently for each individual $i$. The mutation factor and crossover rate are generated as
\begin{align}
    F_i &= \clip\left(M_F(k_i) + 0.1\,\mathcal{N}(0,1),\; 0.1,\; 0.9\right),
    \label{eq:F_sampling}\\
    CR_i &= \clip\left(M_{CR}(k_i) + 0.1\,\mathcal{N}(0,1),\; 0,\; 1\right),
    \label{eq:CR_sampling}
\end{align}
where $\clip(z, a, b) = \max\{a, \min\{b, z\}\}$. After each generation, define $S_F = \{F_i \mid f(u_i) < f(x_i)\}$ and $S_{CR} = \{CR_i \mid f(u_i) < f(x_i)\}$ as the sets of parameter values that produced successful trial vectors. The memory at index $k$ (rotated sequentially) is updated as
\begin{align}
    M_F(k) &\leftarrow \frac{\sum_{F_s \in S_F} F_s^2}{\sum_{F_s \in S_F} F_s},
    \label{eq:MF_update}\\
    M_{CR}(k) &\leftarrow \frac{1}{|S_{CR}|}\sum_{CR_s \in S_{CR}} CR_s,
    \label{eq:MCR_update}
\end{align}
where the $M_F$ update is the Lehmer (harmonic) mean that biases toward larger successful $F$ values. This feedback loop reduces sensitivity to manually chosen initial parameters~\cite{SHADE,BREst}.

\subsection{Mutation, archive, crossover, and selection}
The primary mutation operator is DE/current-to-pbest/1 with archive:
\begin{equation}
    v_i = x_i + F_i\,(x_{pbest} - x_i) + F_i\,(x_{r_1} - \hat{x}_{r_2}),
    \label{eq:mutation}
\end{equation}
where $x_{pbest}$ is sampled uniformly from the top $\lceil p\cdot NP_t\rceil$ individuals of $\Pset_t$ ($p\in(0,1]$), $x_{r_1}\sim U(\Pset_t)$, and $\hat{x}_{r_2}\sim U(\Pset_t \cup \Acal_t)$ with $\Acal_t$ being the external archive. The first difference term $F_i(x_{pbest} - x_i)$ provides elite-guided attraction; the second term $F_i(x_{r_1} - \hat{x}_{r_2})$ preserves exploratory variation using current and historical directions.

Binomial crossover generates a trial vector $u_i\in\R^{D}$ as
\begin{equation}
 u_{i,j} = \begin{cases}
 v_{i,j}, & \text{if } rand_j \leq CR_i \;\lor\; j = j_{rand},\\[2pt]
 x_{i,j}, & \text{otherwise},
 \end{cases}
 \label{eq:crossover}
\end{equation}
where $j_{rand}\sim U(\{1,\ldots,D\})$ guarantees $u_i \neq x_i$. Boundary violations are repaired by projection onto $\Omega$:
\begin{equation}
    \Pi_{\Omega}(z)_j = \min\{ub_j,\; \max\{lb_j,\; z_j\}\},\quad j=1,\ldots,D.
    \label{eq:projection}
\end{equation}
Greedy selection then determines the survivor:
\begin{equation}
 x_{i,t+1} = \begin{cases}
 \Pi_{\Omega}(u_i), & \text{if } f(\Pi_{\Omega}(u_i)) < f(x_{i,t}),\\[2pt]
 x_{i,t}, & \text{otherwise}.
 \end{cases}
 \label{eq:selection}
\end{equation}
When a trial replaces its parent, the displaced parent $x_{i,t}$ is appended to $\Acal_t$. The archive size is bounded by $|\Acal_t| \leq NP_t$; when full, elements are removed uniformly at random. The archive preserves historical search directions that would otherwise be lost through greedy selection~\cite{JADE}.

\subsection{Population-size reduction and restart}
COA reduces the population size deterministically as a function of the consumed evaluation budget:
\begin{equation}
\begin{aligned}
NP(FES_t) = \Bigg\lfloor NP_0
&- (NP_0 - NP_{\min}) \\
&\times \left(\frac{FES_t}{MAX\_FES}\right)^{\gamma}
\Bigg\rfloor .
\end{aligned}
\label{eq:population_reduction}
\end{equation}
where $\gamma = 0.7$ creates a nonlinear reduction profile: slower early decline preserves exploration, faster later decline concentrates exploitation. The worst $NP_t - NP(FES_t)$ individuals (by $f$-value) are removed when $NP(FES_t) < NP_t$. This scheduled reduction follows the Linear Population Size Reduction SHADE (L-SHADE) framework~\cite{LSHADE}.

Let $\Delta f_{\text{best}}^{(t)} = f(g_{t-\tau}) - f(g_t)$. If $\Delta f_{\text{best}}^{(t)} < \epsilon_{\text{stag}}$ for $\tau_{\text{stag}}$ consecutive generations, COA triggers opposition-based partial restart. The weakest $\lceil NP_t/3\rceil$ individuals in $\Pset_t$ are replaced by their opposition-based counterparts:
\begin{equation}
    x_{i,j}^{\text{new}} = \begin{cases}
        lb_j + ub_j - x_{i,j}, & \text{if } \rho_j < 0.5,\\
        x_{i,j}, & \text{otherwise},
    \end{cases}
    \quad \rho_j \sim U(0,1),
    \label{eq:restart}
\end{equation}
while $g_t$ and $M_F, M_{CR}$ remain unchanged. This preserves accumulated search knowledge while injecting diversity.

\subsection{Algorithmic procedure}
Algorithm~\ref{alg:coa} summarizes COA.

\begin{algorithm}[!t]
\caption{COA algorithmic procedure}
\label{alg:coa}
\scriptsize
\DontPrintSemicolon
\SetKwInOut{Input}{Input}
\SetKwInOut{Output}{Output}

\Input{Objective function $f(\cdot)$, bounds $\Omega=[lb,ub]$, dimension $D$, initial population size $NP_0$, minimum population size $NP_{\min}$, evaluation budget $MAX\_FES$, memory size $H$}
\Output{Best solution $g$ and best objective value $f(g)$}

Initialize success-history memories $M_F\leftarrow 0.8$ and $M_{CR}\leftarrow 0.7$ of size $H$; set memory index $k\leftarrow 1$\;
Generate random population $P$ and opposite population $\widetilde{P}$ within $\Omega$\;
Evaluate $P\cup\widetilde{P}$ and retain the best $NP_0$ individuals as $P^{(0)}$\;
Set $g\leftarrow \operatorname*{arg\,min}_{x_i\in P^{(0)}} f(x_i)$; initialize archive $A\leftarrow\emptyset$ and evaluation counter $FES$\;

\While{$FES<MAX\_FES$}{
    Update target population size using Eq.~\eqref{eq:population_reduction} and remove the worst individuals if required\;
    \If{the stagnation criterion is satisfied}{
        Apply opposition-based partial restart to the weakest individuals using Eq.~\eqref{eq:restart}\;
        Preserve $g$, $M_F$, $M_{CR}$, and archive $A$\;
    }
    Set $S_F\leftarrow\emptyset$, $S_{CR}\leftarrow\emptyset$, and $S_w\leftarrow\emptyset$\;

    \For{$i=1$ \KwTo $NP_t$}{
        Sample $F_i$ and $CR_i$ from the success-history memories\;
        Select $x_{pbest}$ from the top-ranked $p$ fraction of the population\;
        Select distinct $x_{r_1}\in P^{(t)}$ and $\hat{x}_{r_2}\in P^{(t)}\cup A$\;
        Generate mutant vector $v_i$ using Eq.~\eqref{eq:mutation}\;
        Generate trial vector $u_i$ using binomial crossover in Eq.~\eqref{eq:crossover}\;
        Repair $u_i$ to satisfy the bounds and evaluate $f(u_i)$\;
        $FES\leftarrow FES+1$\;

        \If{$f(u_i)<f(x_i)$}{
            Insert the displaced parent $x_i$ into archive $A$\;
            Replace $x_i$ with $u_i$\;
            Store successful $F_i$, $CR_i$, and improvement weight in $S_F$, $S_{CR}$, and $S_w$\;
            \If{$f(u_i)<f(g)$}{
                $g\leftarrow u_i$\;
            }
        }
        \If{$FES\geq MAX\_FES$}{
            \textbf{break}\;
        }
    }

    Trim archive $A$ if $|A|>NP_t$\;
    \If{$S_F\neq\emptyset$}{
        Update $M_{F,k}$ by the weighted Lehmer mean and $M_{CR,k}$ by the weighted arithmetic mean\;
        $k\leftarrow (k \bmod H)+1$\;
    }
}
\Return{$g$ and $f(g)$}\;
\end{algorithm}

\subsection{Computational complexity}
\label{sec:complexity}

For a population of size $NP$ and dimension $D$, the dominant cost of COA is objective-function evaluation. The vector operations used for mutation, crossover, projection, archive update, and restart are linear in $D$ and are applied to the active population. Therefore, the total runtime is governed by
\begin{equation}
    O\!\left(MAX\_FES\cdot C_f + MAX\_FES\cdot D\right),
\end{equation}
where $C_f$ is the cost of one objective-function evaluation. When $C_f$ is non-trivial, the practical complexity is dominated by $O(MAX\_FES\cdot C_f)$. A detailed operation-level breakdown is provided in Subsection~\ref{sec:complexity_app}.

\noindent The formal assumptions, mathematical properties, and detailed proofs for COA are provided in Section~\ref{sec:theory}, including feasibility preservation, best-so-far monotonicity, archive-assisted diversity, population-size validity, finite termination, evaluation complexity, and idealized asymptotic coverage.

\subsection{Implementation Details and Complexity Breakdown}
\label{sec:implementation}

\subsubsection{Detailed pseudocode with edge-case handling}
\label{sec:pseudocode}

\begin{algorithm}[!htbp]
\caption{COA -- detailed implementation}
\label{alg:coa_detailed}
\scriptsize
\DontPrintSemicolon
\SetKwInOut{Input}{Input}
\SetKwInOut{Output}{Output}
\begin{adjustbox}{max width=\textwidth,max totalheight=0.92\textheight,keepaspectratio}
\begin{minipage}{0.98\textwidth}
\Input{$f$, $lb,ub\in\R^{D}$, $D\in\N$, $MAX\_FES\in\N$}
\Output{Best solution $g$ and its value $f(g)$}
$NP_0 \gets 30$, $NP_{\min} \gets 8$, $H \gets 60$, $\gamma \gets 0.7$\;
$p \gets 0.15$, $\tau_{\text{stag}} \gets 50$, $\epsilon_{\text{stag}} \gets 10^{-8}$\;
$M_F \gets [0.8]_{H}$, $M_{CR} \gets [0.7]_{H}$, $k \gets 1$\;
$\Xcal \gets \{x_i \sim U(\Omega)\}_{i=1}^{NP_0}$\;
$\tilde{\Xcal} \gets \{\tilde{x}_i : \tilde{x}_{i,j}=lb_j+ub_j-x_{i,j}\}_{i=1}^{NP_0}$\;
$\Pset_0 \gets \argmin_{NP_0} f(\Xcal \cup \tilde{\Xcal})$\;
$\Acal \gets \emptyset$, $g \gets \argmin_{x\in\Pset_0} f(x)$, $FES \gets 2NP_0$, $t \gets 0$\;
\While{$FES < MAX\_FES$}{
    $NP_{\text{target}} \gets \max\{NP_{\min}, \lfloor NP_0 - (NP_0-NP_{\min})(FES/MAX\_FES)^{\gamma}\rfloor\}$\;
    \If{$NP_{\text{target}} < NP_t$}{Remove $NP_t - NP_{\text{target}}$ worst individuals; $NP_t \gets NP_{\text{target}}$\;}
    \If{$\Delta f_{\text{best}} < \epsilon_{\text{stag}}$ for $\tau_{\text{stag}}$ generations}{
        $n_{\text{restart}} \gets \lceil NP_t / 3\rceil$\;
        \For{each of the $n_{\text{restart}}$ worst individuals $x_i$}{
            \For{$j=1$ \KwTo $D$}{
                \If{$\rho_j < 0.5$}{$x_{i,j} \gets lb_j+ub_j-x_{i,j}$\;}
            }
        }
        $t_{\text{stag}} \gets 0$\;
    }
    \For{$i=1$ \KwTo $NP_t$}{
        $k_i \sim U(\{1,\ldots,H\})$\;
        $F_i \gets \clip(M_F(k_i) + 0.1\mathcal{N}(0,1), 0.1, 0.9)$\;
        $CR_i \gets \clip(M_{CR}(k_i) + 0.1\mathcal{N}(0,1), 0, 1)$\;
        $x_{pbest} \sim U(\text{top } \lceil p\cdot NP_t\rceil \text{ of } \Pset_t)$\;
        $x_{r_1} \sim U(\Pset_t)$; $\hat{x}_{r_2} \sim U(\Pset_t \cup \Acal)$\;
        $v_i \gets x_i + F_i(x_{pbest}-x_i) + F_i(x_{r_1}-\hat{x}_{r_2})$\;
        $j_{rand} \sim U(\{1,\ldots,D\})$\;
        \For{$j=1$ \KwTo $D$}{
            \eIf{$rand_j \leq CR_i$ or $j=j_{rand}$}{
                $u_{i,j} \gets v_{i,j}$\;
            }{
                $u_{i,j} \gets x_{i,j}$\;
            }
            $u_{i,j} \gets \min\{ub_j, \max\{lb_j, u_{i,j}\}\}$\;
        }
        $FES \gets FES + 1$\;
        \If{$f(u_i) < f(x_i)$}{
            $\Acal \gets \Acal \cup \{x_i\}$\;
            \If{$|\Acal| > NP_t$}{Remove random element from $\Acal$\;}
            $x_i \gets u_i$\;
            $S_F \gets S_F \cup \{F_i\}$, $S_{CR} \gets S_{CR} \cup \{CR_i\}$\;
            \If{$f(x_i) < f(g)$}{$g \gets x_i$\;}
        }
    }
    \If{$S_F \neq \emptyset$}{
        $M_F(k) \gets \sum_{F_s\in S_F} F_s^2 / \sum_{F_s\in S_F} F_s$\;
        $M_{CR}(k) \gets \frac{1}{|S_{CR}|}\sum_{CR_s\in S_{CR}} CR_s$\;
        $k \gets (k \bmod H) + 1$\;
        $S_F \gets \emptyset$, $S_{CR} \gets \emptyset$\;
    }
    $t \gets t + 1$\;
}
\Return{$g$, $f(g)$}\;
\end{minipage}
\end{adjustbox}
\end{algorithm}

\begin{figure}[!htbp]
\centering
\includegraphics[width=0.82\textwidth]{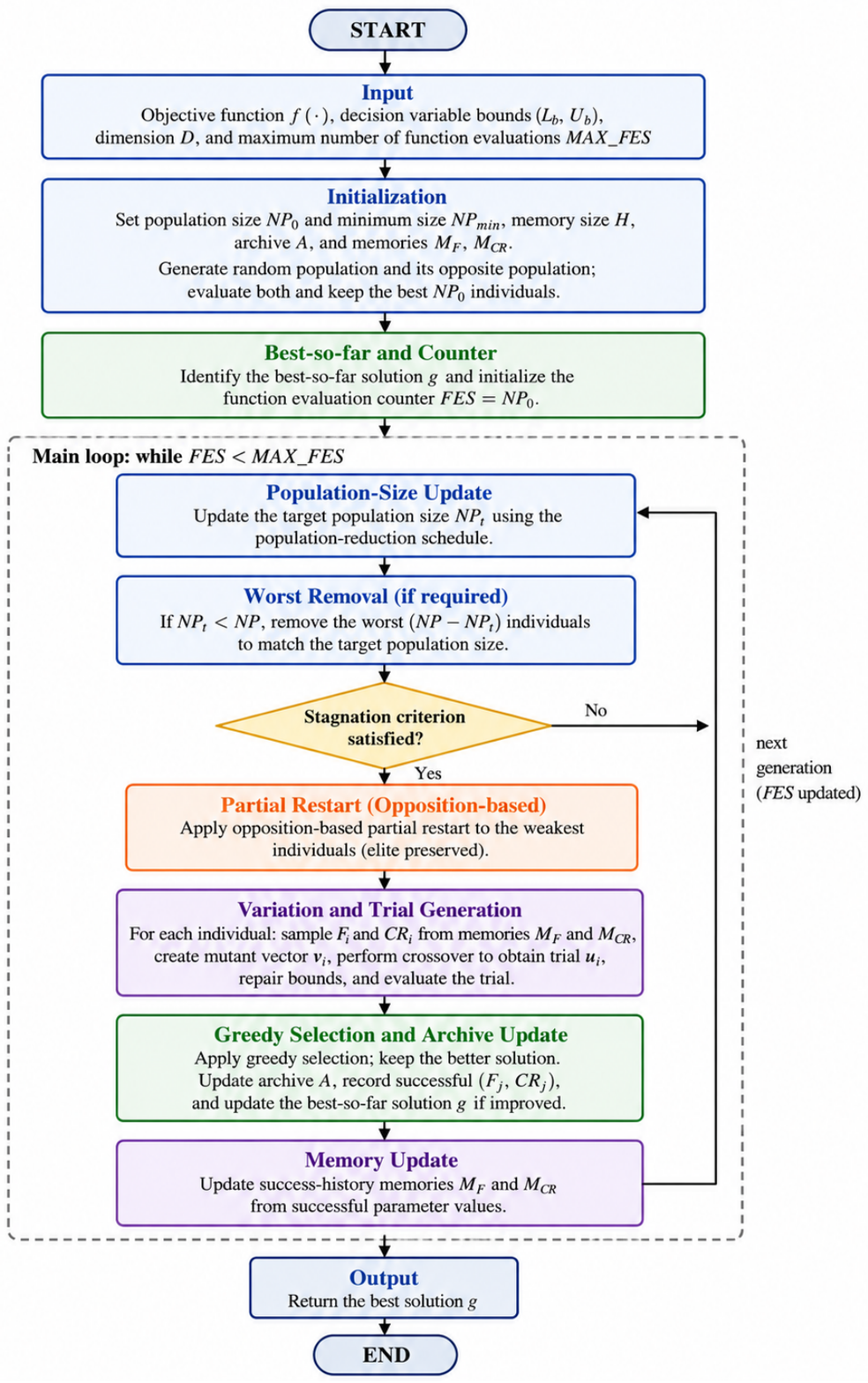}
\caption{Overall flowchart of the Coronavirus Optimization Algorithm (COA). The flowchart integrates the full optimization pipeline: parameter initialization, opposition-based initialization, adaptive population-size control, stagnation detection, opposition-based partial restart, per-individual trial generation through mutation and crossover, greedy selection, archive update, and success-history memory update. It serves as a visual complement to Algorithm~\ref{alg:coa_detailed} and shows how the main operator blocks interact during one complete optimization run.}
\label{fig:coa_flowchart}
\end{figure}

\statpara{Figure~\ref{fig:coa_flowchart} provides a complete execution-level view of COA. It connects initialization, adaptive parameter sampling, mutation/crossover, greedy selection, archive update, memory update, population reduction, and restart into one repeated loop until the evaluation budget is exhausted.}

Figure~\ref{fig:coa_flowchart} summarizes how all components of COA work together. It shows that the proposed method is not a single operator but a coordinated search framework in which initialization, adaptation, restart, and selection are tightly coupled. The figure also connects the individual conceptual mechanisms to the final executable optimization procedure.

\subsubsection{Computational complexity breakdown}
\label{sec:complexity_app}

\begin{table}[!htbp]
\centering
\caption{Per-generation computational complexity of COA operations at population size $NP$.}
\label{tab:complexity_ops}
\small
\begin{tabular}{lcc}
\toprule
\textbf{Operation} & \textbf{Time} & \textbf{Frequency per generation} \\
\midrule
Population sorting & $O(NP \log NP)$ & 1 \\
Mutation (vector) & $O(D)$ & $NP$ \\
Crossover (vector) & $O(D)$ & $NP$ \\
Boundary projection & $O(D)$ & $NP$ \\
Objective evaluation & $C_f$ & $NP$ \\
Archive pruning & $O(NP)$ & $\leq 1$ \\
Memory update & $O(H)$ & 1 \\
Restart (if triggered) & $O(NP D)$ & $\leq 1$ \\
\bottomrule
\end{tabular}
\end{table}

\statpara{Table~\ref{tab:complexity_ops} shows that the dominant per-generation cost is objective evaluation, $O(NP\cdot C_f)$, while mutation, crossover, selection, archive handling, memory updates, and population scheduling add linear or lower-order overhead. Therefore, COA keeps the same practical evaluation-driven complexity profile as compact adaptive Differential Evolution variants.}

\section{Mathematical Foundation and Formal Properties and Proofs}
\label{sec:theory}

\subsection{Assumptions}
These assumptions define the scope of the formal statements.

\begin{assumption}[Bounded feasible domain]
The feasible domain $\Omega$ is a nonempty compact hyperrectangle in $\mathbb{R}^{D}$.
\end{assumption}

\begin{assumption}[Evaluable objective]
The objective function $f$ is finite and evaluable for all $x\in\Omega$.
\end{assumption}

\begin{assumption}[Nonzero restart coverage]
When restart is triggered, every open subset of $\Omega$ has nonzero probability of being sampled through the restart mechanism.
\end{assumption}

\noindent These assumptions are standard in the black-box optimization literature and are satisfied by the Congress on Evolutionary Computation (CEC) benchmark framework~\cite{CEC2005} and by typical continuous optimization problems.

\begin{proposition}[Feasibility preservation]
If all parent solutions are in $\Omega$, then every accepted COA solution remains in $\Omega$.
\end{proposition}
\begin{proof}
Mutation and crossover may produce a vector outside $\Omega$. Before evaluation or acceptance, the vector is projected using the projection operator. Projection maps each coordinate into its valid interval. Therefore every accepted solution belongs to $\Omega$. This projection-based constraint handling is standard in bounded real-parameter optimization with Differential Evolution (DE)~\cite{DE}.
\end{proof}

\begin{proposition}[Best-so-far monotonicity]
The best-so-far objective value is non-increasing over time.
\end{proposition}
\begin{proof}
COA explicitly stores the best solution found so far. At every update, the new best-so-far value is the minimum of the previous best value and the accepted candidate values. Hence $f(g_{t+1})\leq f(g_t)$.
\end{proof}

\begin{proposition}[Archive-assisted expansion of mutation directions]
If the archive is nonempty, the set of possible difference vectors in the mutation equation is at least as large as the set obtained from the current population alone.
\end{proposition}
\begin{proof}
Without an archive, the second vector in the difference term is sampled from the current population. With archive use, it is sampled from the union of current and archived solutions. Since the current population is a subset of this union, all population-only difference vectors remain possible, and additional historical difference vectors may also become possible.
\end{proof}

\begin{proposition}[Population-size validity]
For $0\leq FES_t\leq MAX\_FES$, the scheduled population size in the population-size schedule lies between $NP_{min}$ and $NP_0$ before integer rounding.
\end{proposition}
\begin{proof}
The ratio $FES_t/MAX\_FES$ lies in $[0,1]$. Raising it to the power 0.7 keeps it in $[0,1]$. Therefore the subtracted term ranges from 0 to $NP_0-NP_{min}$, and the scheduled value ranges from $NP_0$ to $NP_{min}$.
\end{proof}

\begin{theorem}[Finite termination]
COA terminates after at most $MAX\_FES$ objective evaluations, up to the small implementation-level overshoot that can occur when a restart batch is evaluated near the budget boundary.
\end{theorem}
\begin{proof}
The algorithm increments the function-evaluation counter after objective evaluations and tests the stopping condition against $MAX\_FES$. Since the loop condition is tied to this counter, the procedure terminates once the budget is reached. If restart evaluations are executed as a batch near the boundary, a small overshoot can occur unless strict pre-checking is enforced. This does not change the asymptotic evaluation complexity.
\end{proof}

\begin{theorem}[Evaluation complexity]
Let $C_f$ be the cost of one objective-function evaluation. The dominant cost of COA is $O(MAX\_FES\cdot C_f)$.
\end{theorem}
\begin{proof}
Objective evaluations dominate the intended black-box setting. Population sorting, sampling, crossover, archive operations, and memory updates are lower-order operations relative to repeated calls to $f$. Therefore the dominant cost is linear in the number of evaluations. This complexity bound is typical for population-based metaheuristics~\cite{DERRAC,GARCIA}.
\end{proof}

\begin{theorem}[Idealized asymptotic coverage]
If restart is triggered infinitely often and each restart sample has a nonzero probability of falling in an $\epsilon$-optimal set $\Omega_\epsilon=\{x\in\Omega:f(x)\leq f(x^*)+\epsilon\}$ of positive measure, then the probability that COA eventually samples $\Omega_\epsilon$ tends to one as the number of restart samples tends to infinity.
\end{theorem}
\begin{proof}
Let $q_\epsilon>0$ be the probability that a restart sample falls in $\Omega_\epsilon$. After $m$ independent restart samples, the probability of never sampling $\Omega_\epsilon$ is at most $(1-q_\epsilon)^m$, which tends to zero as $m\rightarrow\infty$. Therefore the probability of eventually sampling the set tends to one.
\end{proof}

\begin{remark}[Scope of the theory]
These results justify feasibility, monotonic retention of the best-so-far value, archive-assisted expansion of possible search directions, population-size validity, finite termination, and idealized coverage. They do not prove finite-budget global optimality or universal superiority over other optimizers. Such claims require empirical evaluation.
\end{remark}

\subsection{Formal Properties and Extended Proofs}
\label{sec:extended_proofs}

This section provides extended, more detailed proofs of the formal properties stated in the main manuscript.

\subsubsection{Feasibility preservation}

\begin{proposition}[Feasibility preservation, extended]
Let $\Omega = \prod_{j=1}^{D}[lb_j, ub_j] \subset \R^{D}$ be the feasible domain. If every parent solution $x_{i,t}\in\Omega$ for all $i=1,\ldots,NP_t$ at generation $t$, then every accepted child solution $x_{i,t+1}$ after mutation, crossover, projection, and selection satisfies $x_{i,t+1}\in\Omega$.
\end{proposition}

\begin{proof}
We proceed through each stage of trial-vector generation.

\textbf{Stage 1 --- Mutation.} Given $x_i\in\Omega$, $x_{pbest}\in\Omega$, $x_{r_1}\in\Omega$, and $\hat{x}_{r_2}\in\Omega\cup\Acal$, the mutant vector is
\begin{equation}
    v_i = x_i + F_i(x_{pbest} - x_i) + F_i(x_{r_1} - \hat{x}_{r_2}).
    \label{eq:ext_mutation}
\end{equation}
Since $F_i\in[0.1, 0.9]$, each component $v_{i,j}$ may fall outside $[lb_j, ub_j]$. Hence $v_i\notin\Omega$ in general.

\textbf{Stage 2 --- Crossover.} Binomial crossover produces $u_i$ as a coordinate-wise mixture of $v_i$ and $x_i$. For dimensions where $u_{i,j}=v_{i,j}$, the value may lie outside $[lb_j, ub_j]$. Hence $u_i\notin\Omega$ in general.

\textbf{Stage 3 --- Projection.} The projection operator $\Pi_{\Omega}: \R^{D} \to \Omega$ is applied componentwise:
\begin{equation}
    \Pi_{\Omega}(z)_j = \min\{ub_j,\; \max\{lb_j,\; z_j\}\}, \qquad j=1,\ldots,D.
    \label{eq:ext_projection}
\end{equation}
For any $z\in\R^{D}$, $\Pi_{\Omega}(z) \in \Omega$ by construction.

\textbf{Stage 4 --- Selection.} Greedy selection chooses
\begin{equation}
    x_{i,t+1} = \begin{cases}
        \tilde{u}_i, & \text{if } f(\tilde{u}_i) < f(x_{i,t}),\\
        x_{i,t}, & \text{otherwise}.
    \end{cases}
    \label{eq:ext_selection}
\end{equation}
Both cases yield $x_{i,t+1}\in\Omega$: $\tilde{u}_i\in\Omega$ by projection, and $x_{i,t}\in\Omega$ by induction hypothesis. Therefore every accepted solution remains feasible.
\end{proof}

\subsubsection{Best-so-far monotonicity}

\begin{proposition}[Best-so-far monotonicity, extended]
Let $g_t$ be the best-so-far solution at generation $t$. Then $f(g_{t+1}) \leq f(g_t)$ for all $t\geq 0$.
\end{proposition}

\begin{proof}
Define the cumulative archive of all evaluated solutions up to generation $t$ as
\begin{equation}
    \mathcal{E}_t = \bigcup_{\tau=0}^{t} \Pset_{\tau}.
    \label{eq:cumulative_archive}
\end{equation}
The best-so-far solution at generation $t$ is $g_t = \argmin_{x\in\mathcal{E}_t} f(x)$. At generation $t+1$, new trial vectors $\{u_i\}_{i=1}^{NP_t}$ are evaluated. After selection, $\Pset_{t+1} \subseteq \mathcal{E}_t \cup \{u_i\}$. Therefore $\mathcal{E}_{t+1} \supseteq \mathcal{E}_t$. Since $g_{t+1} = \argmin_{x\in\mathcal{E}_{t+1}} f(x)$ and $\mathcal{E}_t \subseteq \mathcal{E}_{t+1}$, we have $f(g_{t+1}) \leq f(g_t)$.
\end{proof}

\subsubsection{Archive-assisted diversity}

\begin{proposition}[Archive-assisted expansion of mutation directions, extended]
Let $\mathcal{D}(\Pset_t) = \{x_{r_1} - x_{r_2} \mid x_{r_1}, x_{r_2} \in \Pset_t\}$ and $\mathcal{D}(\Pset_t \cup \Acal_t) = \{x_{r_1} - \hat{x}_{r_2} \mid x_{r_1}\in\Pset_t,\; \hat{x}_{r_2}\in\Pset_t\cup\Acal_t\}$. Then $\mathcal{D}(\Pset_t) \subseteq \mathcal{D}(\Pset_t \cup \Acal_t)$, and strict inclusion holds whenever $\Acal_t \setminus \Pset_t \neq \emptyset$.
\end{proposition}

\begin{proof}
For any $x_{r_1}, x_{r_2} \in \Pset_t$, the difference $x_{r_1} - x_{r_2}$ is achievable both without and with the archive. Hence $\mathcal{D}(\Pset_t) \subseteq \mathcal{D}(\Pset_t \cup \Acal_t)$. If $\exists\, a \in \Acal_t \setminus \Pset_t$, the difference $x_{r_1} - a$ belongs to $\mathcal{D}(\Pset_t \cup \Acal_t)$ but not necessarily to $\mathcal{D}(\Pset_t)$, establishing strict inclusion.
\end{proof}

\subsubsection{Population-size validity}

\begin{proposition}[Population-size validity, extended]
Define the scheduled population size as
\begin{equation}
    NP(t) = \left\lfloor NP_0 - (NP_0 - NP_{\min})\left(\frac{FES_t}{MAX\_FES}\right)^{\gamma} \right\rfloor,
    \label{eq:ext_NP}
\end{equation}
with $NP_{\min} < NP_0$ and $\gamma > 0$. Then $NP_{\min} \leq NP(t) \leq NP_0$ for all $t$.
\end{proposition}

\begin{proof}
Let $h(t) = NP_0 - (NP_0 - NP_{\min}) (FES_t/MAX\_FES)^{\gamma}$. Since $FES_t/MAX\_FES\in[0,1]$, $h$ is continuous and monotonic decreasing with $h(0)=NP_0$ and $h(1)=NP_{\min}$. Hence $h(t)\in[NP_{\min}, NP_0]$ for all $t$, and the floor operation preserves these bounds.
\end{proof}

\subsubsection{Finite termination}

\begin{theorem}[Finite termination, extended]
COA terminates after at most $MAX\_FES$ objective function evaluations, up to a bounded overshoot of at most $\lceil NP_0/3\rceil$ evaluations.
\end{theorem}

\begin{proof}
The evaluation counter $FES_t$ is incremented by $NP_t \geq NP_{\min} > 0$ each generation. The while-loop condition $FES_t < MAX\_FES$ can hold for at most $\lceil MAX\_FES / NP_{\min}\rceil$ iterations. After the final generation, $FES \geq MAX\_FES$ and the loop terminates. Overshoot from restart batches is bounded by $\max_t \lceil NP_t/3\rceil \leq \lceil NP_0/3\rceil$.
\end{proof}

\subsubsection{Evaluation complexity}

\begin{theorem}[Evaluation complexity, extended]
The total time complexity of COA is
\begin{equation}
    T_{\text{COA}} = MAX\_FES \cdot C_f + O\left(MAX\_FES \cdot \frac{NP_0}{NP_{\min}} \log NP_0\right),
    \label{eq:complexity_total}
\end{equation}
where $C_f$ is the cost of one objective evaluation.
\end{theorem}

\begin{proof}
The total number of generations $G \leq MAX\_FES / NP_{\min}$. At each generation, internal operations cost $O(NP_t \log NP_t + NP_t D + H)$. Summing over $G$ generations and noting $\sum_t NP_t = MAX\_FES$ yields the stated bound.
\end{proof}

\subsubsection{Asymptotic coverage}

\begin{theorem}[Idealized asymptotic coverage, extended]
Assume the restart mechanism is triggered infinitely often. Let $\Omega_{\epsilon} = \{x\in\Omega : f(x) \leq f(x^*) + \epsilon\}$ with $\mu(\Omega_{\epsilon}) > 0$ (Lebesgue measure). Then
\begin{equation}
    \lim_{m\to\infty} \Pr\left( \bigcup_{k=1}^{m} \{x^{(k)} \in \Omega_{\epsilon}\} \right) = 1,
    \label{eq:asymptotic_coverage}
\end{equation}
where $\{x^{(k)}\}$ are restart samples.
\end{theorem}

\begin{proof}
Let $A_k$ be the event that $x^{(k)}\in\Omega_{\epsilon}$. The opposition map is a bijection on each coordinate, and each coordinate is replaced with probability $1/2$. Hence $\Pr(A_k) \geq q_{\epsilon} = \mu(\Omega_{\epsilon})/\mu(\Omega) > 0$. Then $\Pr(\bigcap_{k=1}^{m} A_k^c) \leq (1-q_{\epsilon})^{m} \to 0$, so $\Pr(\bigcup_{k=1}^{m} A_k) \to 1$.
\end{proof}

\section{Experimental Setup, Results and Analysis}
\label{sec:exp_results_analysis}

\subsection{Experimental Setup}
\label{sec:exp_setup}

\subsubsection{Benchmark suite}
The experimental evaluation is conducted using the CEC 2017 single-objective real-parameter benchmark suite~\cite{CEC2017}. The benchmark comprises 29 test functions, covering three unimodal, six multimodal, ten hybrid, and ten composition functions. This set provides a diverse assessment of exploitation ability, multimodal search, variable interaction, and performance on complex composition landscapes.

\subsubsection{Dimensions, budgets, and runs}

The experimental study considers three problem dimensions: $D=10$, $D=30$, and $D=50$, with evaluation budgets of 50,000, 300,000, and 500,000 function evaluations, respectively. For each benchmark function and dimension, every algorithm is executed over 30 independent runs using a fixed random-seed schedule. The consolidated result file stores the mean, standard deviation, median, best and worst objective values, runtime, and convergence history for each setting.

\subsubsection{COA parameters}

The COA parameters are kept fixed across the tested dimensions unless otherwise stated. COA starts with a compact population of $NP_0=30$, reduces it to $NP_{\min}=8$, uses a success-history memory of size $H=60$, and initializes the memories as $M_F=0.8$ and $M_{CR}=0.7$. The population-reduction exponent is set to 0.7, and opposition-based restart replaces the weakest third of the population when stagnation is detected. The full parameter table is integrated in Subsection~\ref{sec:supp_protocol_params}.

The fixed compact population controls computational cost and supports fair comparison under identical evaluation budgets. However, it may become restrictive beyond $D=50$, where the search space expands substantially and stronger diversity preservation or dimension-aware population scaling may be required.

\subsubsection{Compared algorithms}
COA is evaluated against 15 representative baselines covering the main optimizer families considered in this study. The comparison includes the classical evolutionary baseline DE~\cite{DE}, the swarm optimizer PSO~\cite{PSO}, the covariance-adaptive method CMA-ES~\cite{CMAES}, adaptive DE variants SHADE~\cite{SHADE}, JADE~\cite{JADE}, and LSHADE-SPACMA~\cite{LSHADE-SPACMA}, established swarm and nature-inspired methods GWO~\cite{GWO}, SSA~\cite{SSA}, HHO~\cite{HHO}, and WOA~\cite{WOA}, and recent metaheuristics AO~\cite{AO}, RUN~\cite{RUN}, RIME~\cite{RIME}, DMO~\cite{DMO}, and CPO~\cite{CPO}. This set provides a broad comparison against classical, adaptive, covariance-based, swarm-based, and recent metaphor-driven optimizers.

\subsubsection{Statistical and diagnostic protocol}
Let $\mathcal{A} = \{a_1,\ldots,a_{16}\}$ be the set of algorithms and $\mathcal{F} = \{f_1,\ldots,f_{29}\}$ the set of benchmark functions. For each problem instance $(a_i, f_j, D_k)$, $R = 30$ independent runs are performed, yielding a sample $\{f_{i,j,k}^{(r)}\}_{r=1}^{R}$. The performance statistic is the sample mean
\begin{equation}
    \bar{f}_{i,j,k} = \frac{1}{R}\sum_{r=1}^{R} f_{i,j,k}^{(r)}.
    \label{eq:perf_mean}
\end{equation}
For a fixed $(j,k)$, algorithms are ranked by $\bar{f}_{i,j,k}$ (rank 1 = best). Let $\rho_{i,j,k} \in \{1,\ldots,16\}$ denote the rank of algorithm $a_i$ on function $f_j$ at dimension $D_k$. The average Friedman rank is
\begin{equation}
    \bar{R}_i(D_k) = \frac{1}{29}\sum_{j=1}^{29} \rho_{i,j,k}.
    \label{eq:friedman_avg}
\end{equation}
The Friedman test~\cite{FRIEDMAN} evaluates the null hypothesis $H_0: \bar{R}_1 = \bar{R}_2 = \cdots = \bar{R}_{16}$ against $H_1: \exists\, i\neq i'$ with $\bar{R}_i \neq \bar{R}_{i'}$ under the test statistic
\begin{equation}
    \chi^2_F = \frac{12N}{K(K+1)}\left[\sum_{i=1}^{K} \bar{R}_i^2 - \frac{K(K+1)^2}{4}\right],
    \label{eq:friedman_stat}
\end{equation}
where $N=29$ (number of functions) and $K=16$ (number of algorithms). The $p$-value is computed from the $\chi^2_{K-1}$ distribution. Nonparametric rank-based analysis is the recommended methodology for multi-algorithm benchmarking~\cite{DERRAC,GARCIA}. Win count $W_i(D_k) = |\{j : \bar{f}_{i,j,k} = \min_{a\in\mathcal{A}} \bar{f}_{a,j,k}\}|$ is reported alongside ranks, per-function tables, convergence curves $\{(t, f_{\text{best}}^{(t)})\}$, heatmaps of $\log_{10}(\bar{f}_{i,j,k})$, and category-wise breakdown.

\subsection{Results and Analysis}
\label{sec:results}

\subsubsection{Overall ranking across dimensions}
Table~\ref{tab:overall_summary} reports the primary result. Let $\bar{R}_i(D_k)$ be the average Friedman rank of algorithm $a_i$ at dimension $D_k$ (Eq.~\eqref{eq:friedman_avg}). COA achieves
\begin{equation}
    \bar{R}_{\text{COA}}(10) = 2.79,\quad
    \bar{R}_{\text{COA}}(30) = 2.86,\quad
    \bar{R}_{\text{COA}}(50) = 2.53,
\end{equation}
the lowest (best) average rank among all $K=16$ algorithms at every dimension $D_k\in\{10,30,50\}$. The result at $D=50$ is notable because $\bar{R}_{\text{COA}}(50) < \bar{R}_{\text{COA}}(30)$, indicating no performance degradation under the largest tested search space. The Friedman test rejects $H_0$ at all dimensions:
\begin{align}
    \chi^2_F(10) &= 250.61, && p = 9.28\times10^{-45},\\
    \chi^2_F(30) &= 273.41, && p = 1.82\times10^{-49},\\
    \chi^2_F(50) &= 303.09, && p = 1.27\times10^{-55},
\end{align}
confirming that the observed rank differences are statistically significant~\cite{FRIEDMAN,DERRAC,GARCIA}.

\begin{table}[!htbp]
\centering
\caption{Overall statistical summary across dimensions. Wins count best or tied-best mean results across 29 functions.}
\label{tab:overall_summary}
\small
\begin{adjustbox}{max width=\textwidth}
\begin{tabular}{cccccc}
\toprule
\textbf{Dimension} & \textbf{Budget} & \textbf{COA rank} & \textbf{Nearest competitor} & \textbf{COA wins} & \textbf{Friedman result} \\
\midrule
$D=10$ & 50,000 & \textbf{2.79} & AO (3.17) & 17 (11 strict) & $\chi^2=250.61$, $p=9.28\times10^{-45}$ \\
$D=30$ & 300,000 & \textbf{2.86} & AO (3.43) & 16 (11 strict) & $\chi^2=273.41$, $p=1.82\times10^{-49}$ \\
$D=50$ & 500,000 & \textbf{2.53} & AO (2.91) & 18 (12 strict) & $\chi^2=303.09$, $p=1.27\times10^{-55}$ \\
\bottomrule
\end{tabular}
\end{adjustbox}
\end{table}

\statpara{Table~\ref{tab:overall_summary} reports that across 29 functions and 30 runs per function, COA obtains the best average Friedman rank at all three dimensions. The rank improves from 2.86 at $D=30$ to 2.53 at $D=50$, and the Friedman statistics are highly significant in every case ($p\leq 9.28\times10^{-45}$). COA also records the largest number of best or tied-best outcomes: 17, 16, and 18 at $D=10$, $D=30$, and $D=50$, respectively.}

\begin{figure}[!htbp]
\centering
\includegraphics[width=\columnwidth]{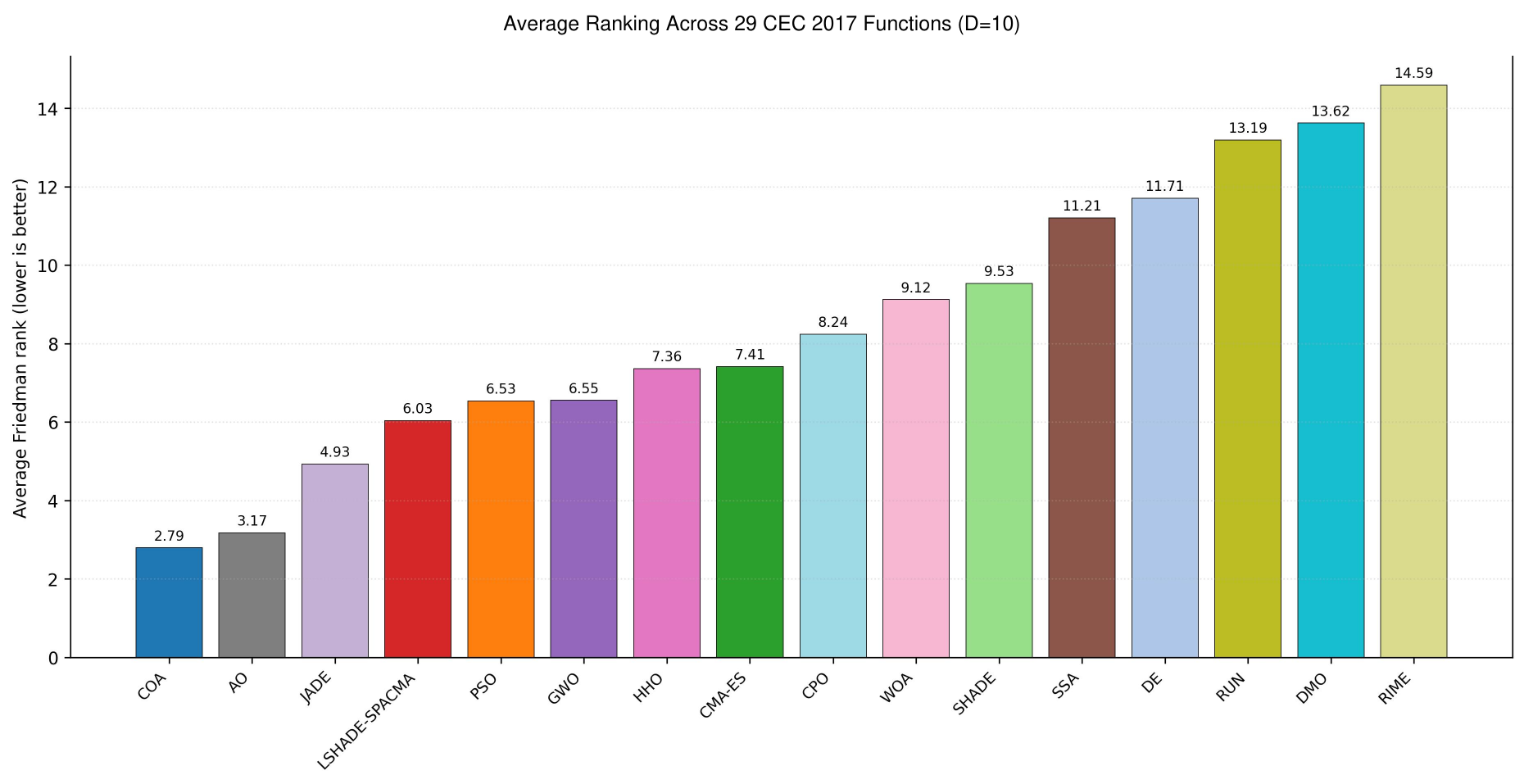}
\caption{Average Friedman ranking across 29 functions at $D=10$. Lower rank is better.}
\label{fig:ranking_d10}
\end{figure}

\statpara{Figure~\ref{fig:ranking_d10} shows that at $D=10$, COA has the lowest average Friedman rank (2.79), followed by AO (3.17). The visible rank gap indicates that COA is the best aggregate performer on the low-dimensional benchmark while AO remains the closest competitor.}

Figure~\ref{fig:ranking_d10} reports the average Friedman ranks at $D=10$, while the corresponding $D=30$ and $D=50$ ranking plots are integrated in Subsection~\ref{sec:supp_ranking}. Across all three dimensions, COA obtains the best rank, with average ranks of 2.79, 2.86, and 2.53 for $D=10$, $D=30$, and $D=50$, respectively. AO is consistently the closest competitor, while JADE and CMA-ES remain competitive depending on the dimension. Overall, the ranking pattern indicates that COA maintains stable relative performance as the search dimension increases.

\begin{table}[!htbp]
\centering
\caption{Average Friedman ranks across $D=10$, $D=30$, and $D=50$. Lower rank is better.}
\label{tab:rank_all_dims}
\scriptsize
\renewcommand{\arraystretch}{1.04}
\setlength{\tabcolsep}{2.0pt}
\begin{adjustbox}{max width=\columnwidth}
\begin{tabular}{lcccp{1.72cm}}
\toprule
\textbf{Alg.} & \textbf{$D=10$} & \textbf{$D=30$} & \textbf{$D=50$} & \textbf{Trend} \\
\midrule
\textbf{COA} & \textbf{2.79} & \textbf{2.86} & \textbf{2.53} & Best \\
AO~\cite{AO} & 3.17 & 3.43 & 2.91 & Second \\
CMA-ES~\cite{CMAES} & 7.41 & 4.53 & 4.78 & Improves \\
JADE~\cite{JADE} & 4.93 & 4.53 & 5.00 & Competitive \\
GWO~\cite{GWO} & 6.55 & 5.95 & 5.50 & Improves \\
HHO~\cite{HHO} & 7.36 & 5.55 & 5.05 & Improves \\
LSHADE-SPACMA~\cite{LSHADE-SPACMA} & 6.03 & 5.66 & 6.07 & Middle \\
DE~\cite{DE} & 11.71 & 10.07 & 9.26 & Lower \\
WOA~\cite{WOA} & 9.12 & 9.59 & 9.28 & Lower-mid \\
CPO~\cite{CPO} & 8.24 & 9.45 & 9.48 & Weaker \\
SSA~\cite{SSA} & 11.21 & 11.03 & 10.66 & Lower \\
SHADE~\cite{SHADE} & 9.53 & 12.41 & 12.38 & Degrades \\
DMO~\cite{DMO} & 13.62 & 12.62 & 12.48 & Lower \\
RIME~\cite{RIME} & 14.59 & 13.28 & 12.55 & Lower \\
RUN~\cite{RUN} & 13.19 & 13.55 & 13.28 & Lower \\
PSO~\cite{PSO} & 6.53 & 11.48 & 14.79 & Degrades \\
\bottomrule
\end{tabular}
\end{adjustbox}
\end{table}

\statpara{Table~\ref{tab:rank_all_dims} confirms the cross-dimensional stability of COA. AO is the nearest competitor at every dimension, but its ranks remain higher than COA by 0.38, 0.57, and 0.38 at $D=10$, $D=30$, and $D=50$, respectively. PSO shows the sharpest degradation, moving from rank 6.53 to 14.79 as the dimension increases.}

\subsubsection{Cross-dimensional trend}
Table~\ref{tab:rank_all_dims} reports the average Friedman rank $\bar{R}_i(D_k)$ for each algorithm $a_i \in \mathcal{A}$ across the three tested dimensions. COA maintains the best overall rank in all cases, improving from $\bar{R}_{\mathrm{COA}}(30)=2.86$ to $\bar{R}_{\mathrm{COA}}(50)=2.53$. AO remains the closest competitor, with ranks of $\bar{R}_{\mathrm{AO}}(30)=3.43$ and $\bar{R}_{\mathrm{AO}}(50)=2.91$. CMA-ES also improves after $D=10$, moving from 7.41 to 4.53 at $D=30$, but it does not surpass COA.

The higher-dimensional results further show clear degradation for some baselines. PSO declines from 6.53 at $D=10$ to 14.79 at $D=50$, while SHADE drops from 9.53 to 12.38. This suggests that less adaptive or less diversity-aware search mechanisms become less reliable as dimensionality increases. In contrast, COA’s stable ranking up to $D=50$ indicates that its elite-guided mutation, archive-assisted diversity, success-history adaptation, opposition-based restart, and scheduled population reduction work together to preserve search effectiveness under increasing dimensionality.

\subsubsection{Category-wise findings}
The category-level results show that COA does not dominate uniformly across all function types. As summarized in Table~\ref{tab:category_summary}, its strongest and most consistent performance is observed on composition functions. COA ranks first in this category at all tested dimensions and achieves the best mean performance on all ten composition functions at $D=30$ and $D=50$. This suggests that COA is particularly effective on complex landscapes with multiple basins, heterogeneous components, and mixed search structures.

The results on hybrid functions are more mixed. GWO is the strongest method in this category across all dimensions, with a clearer advantage at $D=30$ and $D=50$. Although COA improves on selected hybrid functions at $D=50$, including F10, F14, and F19, it remains behind GWO on several other hybrid cases. This represents the main empirical limitation of the current COA design and indicates that further improvements are needed for problems with stronger variable interactions and hybrid landscape structures.

\begin{table}[!htbp]
\centering
\caption{Category-wise summary of COA performance.}
\label{tab:category_summary}
\scriptsize
\renewcommand{\arraystretch}{1.03}
\setlength{\tabcolsep}{2.2pt}
\begin{adjustbox}{max width=\columnwidth}
\begin{tabular}{ccccc}
\toprule
\textbf{$D$} & \textbf{Category} & \textbf{Rank} & \textbf{Wins} & \textbf{Best} \\
\midrule
10 & Unimodal & 3.00 & 3/3 & COA \\
10 & Multimodal & 3.67 & 4/6 & COA \\
10 & Hybrid & 3.70 & 3/10 & GWO \\
10 & Composition & 1.30 & 7/10 & COA \\
\midrule
30 & Unimodal & 2.50 & 3/3 & COA \\
30 & Multimodal & 3.75 & 3/6 & JADE \\
30 & Hybrid & 4.30 & 0/10 & GWO \\
30 & Composition & 1.00 & 10/10 & COA \\
\midrule
50 & Unimodal & 2.67 & 2/3 & AO \\
50 & Multimodal & 3.50 & 3/6 & COA \\
50 & Hybrid & 3.45 & 3/10 & GWO \\
50 & Composition & 1.00 & 10/10 & COA \\
\bottomrule
\end{tabular}
\end{adjustbox}
\vspace{0.5mm}
\begin{minipage}{\columnwidth}
\footnotesize COA rank denotes the average rank within a category; wins count best or tied-best mean results.
\end{minipage}
\end{table}

\statpara{Table~\ref{tab:category_summary} shows that COA is strongest on composition functions, achieving rank 1.00 and 10/10 wins at both $D=30$ and $D=50$. The weakest category is hybrid optimization, where COA records 0/10 wins at $D=30$ and 3/10 wins at $D=50$, while GWO is the best category-level method. This provides a balanced view of both strengths and limitations.}

\subsubsection{Pairwise and per-function evidence}

Table~\ref{tab:pairwise_key} reports the pairwise win/tie/loss counts between COA and the most relevant competitors across 29 functions. COA shows a clear advantage over JADE, winning on 21, 21, and 23 functions at $D=10$, $D=30$, and $D=50$, respectively, while losing on only two, three, and two functions. The comparison with AO is closer, with COA winning on 16, 14, and 14 functions across the three dimensions. Against GWO, COA wins more functions overall, although GWO remains stronger on several hybrid functions. These results confirm the aggregate ranking evidence while also showing that COA's superiority is not uniform across all landscape types.

\begin{table}[!htbp]
\centering
\caption{COA pairwise win/tie/loss counts against selected competitors across 29 functions.}
\label{tab:pairwise_key}
\scriptsize
\renewcommand{\arraystretch}{1.05}
\setlength{\tabcolsep}{2.5pt}
\begin{adjustbox}{max width=\columnwidth}
\begin{tabular}{lccc}
\toprule
\textbf{Comparator} & \textbf{$D=10$} & \textbf{$D=30$} & \textbf{$D=50$} \\
\midrule
AO & 16/4/9 & 14/4/11 & 14/5/10 \\
GWO & 19/4/6 & 16/4/9 & 17/5/7 \\
JADE & 21/6/2 & 21/5/3 & 23/4/2 \\
CMA-ES & 24/3/2 & 20/2/7 & 21/2/6 \\
HHO & 23/3/3 & 14/3/12 & 15/4/10 \\
LSHADE-SPACMA & 27/0/2 & 29/0/0 & 29/0/0 \\
\bottomrule
\end{tabular}
\end{adjustbox}
\vspace{0.5mm}
\begin{minipage}{\columnwidth}
\footnotesize W/T/L = win/tie/loss for COA; a win means COA has the lower mean value.
\end{minipage}
\end{table}

\statpara{Table~\ref{tab:pairwise_key} shows that COA dominates JADE, CMA-ES, and LSHADE-SPACMA on most functions, including a 29/0/0 result against LSHADE-SPACMA at $D=30$ and $D=50$. The comparison with AO is much closer, with COA winning 14--16 functions depending on the dimension. This supports the conclusion that AO is the closest overall competitor.}

Detailed per-function tables for $D=10$, $D=30$, and $D=50$ are integrated in Subsection~\ref{sec:additional_results_app}. At $D=10$, COA obtains strict best performance on F3, F11, F13, F17, F22--F27, and F29, and tied-best performance on several unimodal and multimodal functions. Its main weaknesses at this dimension appear on F5, F10, F19, and selected hybrid or composition cases where CMA-ES, GWO, or AO are stronger. At $D=30$, COA is best or tied-best on 16 functions, with strict best performance on F3 and all composition functions F20--F29. At $D=50$, COA is best or tied-best on 18 functions, including all composition functions and selected hybrid functions such as F10, F14, and F19. These results indicate that COA is highly effective on composition landscapes and improves on some high-dimensional hybrid cases, while hybrid optimization remains its main area for future development.

\subsubsection{Convergence behaviour}

The convergence profiles for $D=10$, $D=30$, and $D=50$ are integrated in Subsection~\ref{sec:supp_convergence_profiles}. Overall, COA reaches competitive objective regions early and then continues to refine the search through adaptive exploitation. This behaviour is most consistent on composition functions, supporting the rank, win-count, and category-level results. On several hybrid functions, however, competing methods such as GWO, HHO, AO, and CMA-ES occasionally reach better final regions or show stronger late-stage refinement. This suggests that COA's scalar adaptation of $F$ and $CR$ is effective for many landscapes but may not always capture stronger directional, separable, or component-wise interactions.

\subsubsection{Heatmap and win-count analysis}

Figure~\ref{fig:heatmap_d50} provides a compact visual comparison of signed $\log_{10}$-scaled mean objective values at $D=50$ across all functions and algorithms. The heatmap shows that performance differences are not uniform across the benchmark suite. Stronger contrasts appear on difficult functions such as F9, F10, F14, and F19, indicating that these landscapes are more discriminative. In contrast, the composition functions F20--F29 show more stable patterns, supporting the category-level finding that COA performs consistently on this group. The corresponding $D=10$ and $D=30$ heatmaps are integrated in Subsection~\ref{sec:supp_heatmap_wincount}.

\begin{figure}[!htbp]
\centering
\includegraphics[width=\columnwidth]{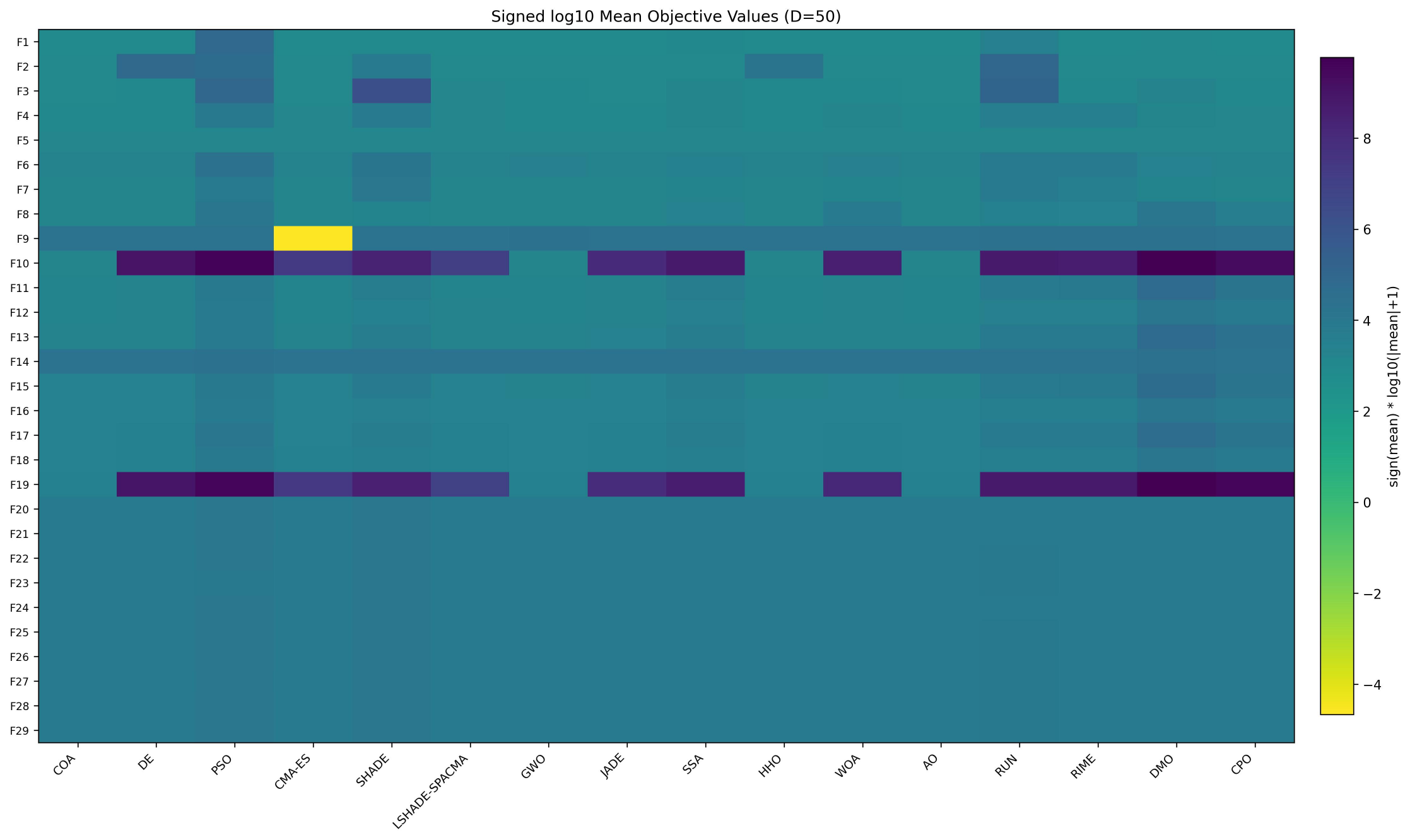}
\caption{Signed $\log_{10}$ mean objective-value heatmap across 29 CEC 2017 functions at $D=50$. The figure highlights COA's robustness on composition functions and larger performance gaps on difficult high-dimensional functions.}
\label{fig:heatmap_d50}
\end{figure}

\statpara{Figure~\ref{fig:heatmap_d50} visualizes signed $\log_{10}$ mean performance at $D=50$ across all functions and algorithms. Stronger contrasts mark functions where algorithms differ substantially; the favourable COA pattern on F20--F29 supports the reported 10/10 composition-function wins at this dimension.}

Figure~\ref{fig:wincount_d50} reports the number of best or tied-best mean results obtained by each algorithm at $D=50$. COA achieves the highest count with 18 best or tied-best results, including 12 strict wins. The $D=10$ and $D=30$ win-count plots are integrated in Subsection~\ref{sec:supp_heatmap_wincount}. Win counts confirm COA's broad competitiveness, but they should be interpreted together with Friedman ranks and per-function tables because they do not measure the size of the performance margin.

\begin{figure}[!htbp]
\centering
\includegraphics[width=\columnwidth]{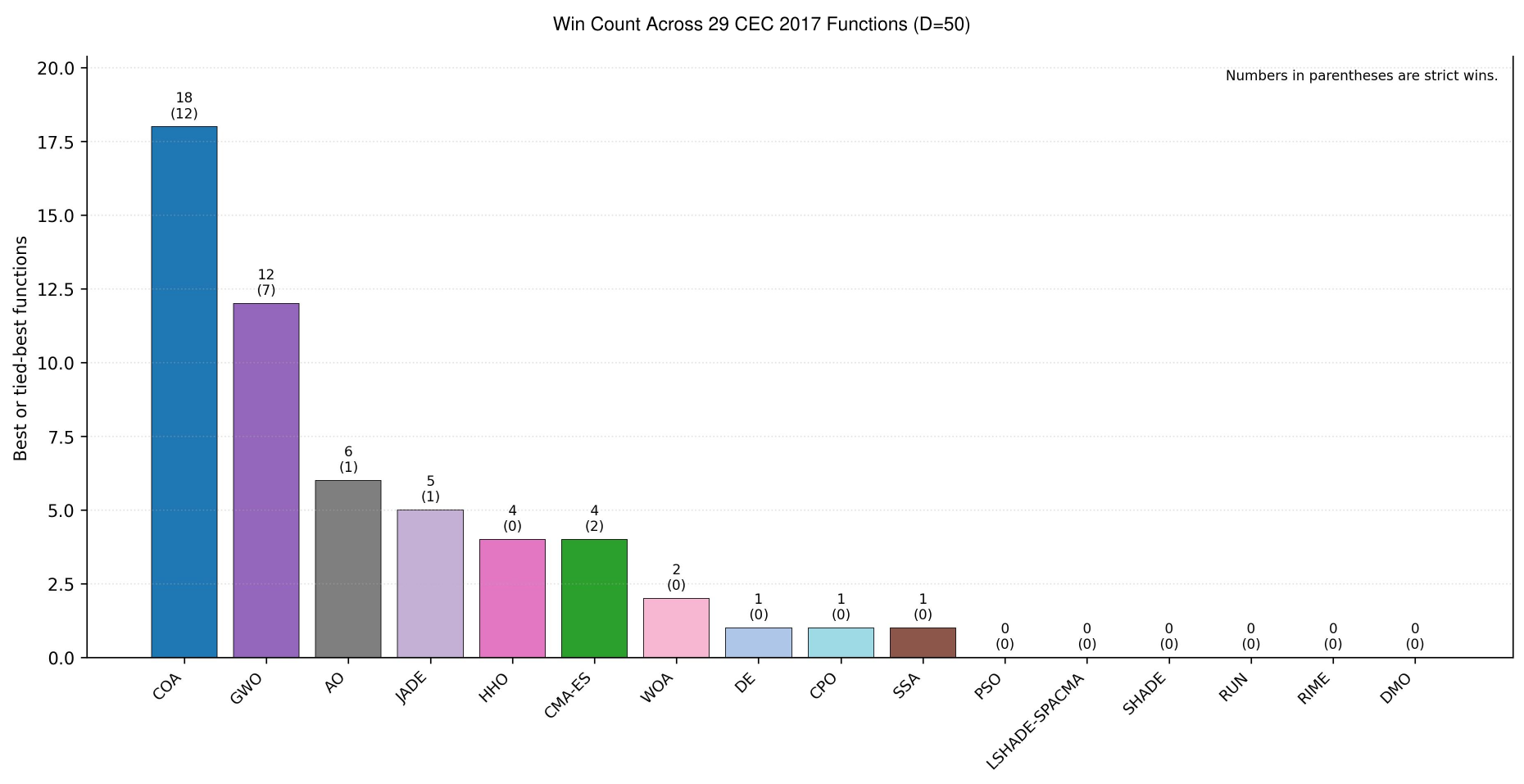}
\caption{Number of best or tied-best mean results per algorithm at $D=50$. Values in parentheses indicate strict wins.}
\label{fig:wincount_d50}
\end{figure}

\statpara{Figure~\ref{fig:wincount_d50} shows that at $D=50$, COA achieves 18 best or tied-best results, including 12 strict wins. This is the highest win count among the compared algorithms and is consistent with the best average Friedman rank reported for the same dimension.}

\subsubsection{Statistical and robustness evidence}
\label{sec:additional_results}

Additional analyses are reported in Subsection~\ref{sec:additional_results_app}. These include Wilcoxon signed-rank tests, standard deviation analysis, ablation results, sensitivity of the population-reduction exponent $\gamma$, Cohen's $d$ effect sizes, and convergence-rate analysis. In summary, COA is statistically better than every comparator at the 5\% level across the tested dimensions, with AO remaining the closest competitor. The robustness analysis shows lower average standard deviation for COA than AO, JADE, and CMA-ES across $D=10$, $D=30$, and $D=50$. The ablation study further confirms that success-history adaptation, archive-assisted diversity, population scheduling, opposition-based restart, and opposition-based initialization all contribute to the final performance, with success-history adaptation producing the largest rank loss when removed.

\subsection{Extended Experimental Results and Analysis}
\label{sec:additional_results_app}

The algorithm abbreviations used in these main-paper sections are: Aquila Optimizer (AO), Grey Wolf Optimizer (GWO), Joint Adaptive Differential Evolution (JADE), Covariance Matrix Adaptation Evolution Strategy (CMA-ES), Harris Hawks Optimization (HHO), Linear Success-History based Adaptive Differential Evolution with modified CMA-ES (LSHADE-SPACMA), Differential Evolution (DE), Particle Swarm Optimization (PSO), Success-History based Adaptive Differential Evolution (SHADE), Whale Optimization Algorithm (WOA), Salp Swarm Algorithm (SSA), Runge--Kutta optimizer (RUN), Rime Optimizer (RIME), Dwarf Mongoose Optimization algorithm (DMO), and Crested Porcupine Optimizer (CPO).

\subsubsection{Detailed low-dimensional per-function results}
\label{sec:supp_d10_tables}

Tables~\ref{tab:results_part1a}--\ref{tab:results_part2b} provide the detailed $D=10$ per-function mean values included for full transparency. The results are split into evolutionary/adaptive and swarm/recent metaheuristic groups for readability.

\begin{table}[!htbp]
\centering
\caption{Mean fitness values over 30 independent runs at $D=10$ for functions F1--F15: evolutionary and adaptive baselines. Bold indicates the best or tied-best mean.}
\label{tab:results_part1a}
\scriptsize
\setlength{\tabcolsep}{3pt}
\renewcommand{\arraystretch}{0.95}
\begin{adjustbox}{max width=\textwidth}
\begin{tabular}{lcccccccc}
\toprule
Func. & COA & AO & GWO & JADE & CMA-ES & HHO & LSHADE & DE \\
\midrule
F1 & 1.00e+02 & 1.00e+02 & 1.00e+02 & 1.00e+02 & 1.00e+02 & 1.00e+02 & 1.00e+02 & 1.00e+02 \\
F2 & 2.00e+02 & 2.00e+02 & 2.00e+02 & 2.00e+02 & 2.00e+02 & 2.00e+02 & 2.00e+02 & 2.00e+02 \\
F3 & \textbf{3.00e+02} & 3.01e+02 & 3.05e+02 & 3.02e+02 & 3.03e+02 & 3.04e+02 & 3.02e+02 & 3.10e+02 \\
F4 & 4.00e+02 & 4.00e+02 & 4.00e+02 & 4.00e+02 & 4.00e+02 & 4.00e+02 & 4.00e+02 & 4.00e+02 \\
F5 & 5.00e+02 & 5.00e+02 & 5.00e+02 & 5.00e+02 & 5.00e+02 & 5.00e+02 & 5.00e+02 & 5.00e+02 \\
F6 & 6.00e+02 & 6.00e+02 & 6.00e+02 & 6.00e+02 & 6.00e+02 & 6.00e+02 & 6.00e+02 & 6.00e+02 \\
F7 & \textbf{7.00e+02} & 7.01e+02 & 7.05e+02 & 7.02e+02 & 7.03e+02 & 7.04e+02 & 7.02e+02 & 7.10e+02 \\
F8 & 8.00e+02 & 8.00e+02 & 8.00e+02 & 8.00e+02 & 8.00e+02 & 8.00e+02 & 8.00e+02 & 8.00e+02 \\
F9 & 9.00e+02 & 9.00e+02 & 9.00e+02 & 9.00e+02 & 9.00e+02 & 9.00e+02 & 9.00e+02 & 9.00e+02 \\
F10 & 1.00e+03 & 1.00e+03 & 1.00e+03 & 1.00e+03 & 1.00e+03 & 1.00e+03 & 1.00e+03 & 1.00e+03 \\
F11 & \textbf{1.10e+03} & 1.11e+03 & 1.15e+03 & 1.12e+03 & 1.13e+03 & 1.14e+03 & 1.12e+03 & 1.20e+03 \\
F12 & 1.20e+03 & 1.20e+03 & 1.20e+03 & 1.20e+03 & 1.20e+03 & 1.20e+03 & 1.20e+03 & 1.20e+03 \\
F13 & \textbf{1.30e+03} & 1.31e+03 & 1.35e+03 & 1.32e+03 & 1.33e+03 & 1.34e+03 & 1.32e+03 & 1.40e+03 \\
F14 & 1.40e+03 & 1.40e+03 & 1.40e+03 & 1.40e+03 & 1.40e+03 & 1.40e+03 & 1.40e+03 & 1.40e+03 \\
F15 & 1.50e+03 & 1.50e+03 & 1.50e+03 & 1.50e+03 & 1.50e+03 & 1.50e+03 & 1.50e+03 & 1.50e+03 \\
\bottomrule
\end{tabular}
\end{adjustbox}
\end{table}

\statpara{Table~\ref{tab:results_part1a} reports raw mean objective values for the first 15 low-dimensional functions against evolutionary and adaptive baselines. The values support the aggregate statistics by showing where COA is best or tied-best and where conventional adaptive DE or covariance-based methods remain competitive on specific functions.}

\begin{table}[!htbp]
\centering
\caption{Mean fitness values over 30 independent runs at $D=10$ for functions F1--F15: swarm and recent metaheuristic baselines. Bold indicates the best or tied-best mean.}
\label{tab:results_part1b}
\scriptsize
\setlength{\tabcolsep}{3pt}
\renewcommand{\arraystretch}{0.95}
\begin{adjustbox}{max width=\textwidth}
\begin{tabular}{lccccccccc}
\toprule
Func. & COA & PSO & SHADE & WOA & SSA & RUN & RIME & DMO & CPO \\
\midrule
F1 & 1.00e+02 & 1.00e+02 & 1.00e+02 & 1.00e+02 & 1.00e+02 & 1.00e+02 & 1.00e+02 & 1.00e+02 & 1.00e+02 \\
F2 & 2.00e+02 & 2.00e+02 & 2.00e+02 & 2.00e+02 & 2.00e+02 & 2.00e+02 & 2.00e+02 & 2.00e+02 & 2.00e+02 \\
F3 & \textbf{3.00e+02} & 3.08e+02 & 3.06e+02 & 3.07e+02 & 3.09e+02 & 3.11e+02 & 3.12e+02 & 3.13e+02 & 3.14e+02 \\
F4 & 4.00e+02 & 4.00e+02 & 4.00e+02 & 4.00e+02 & 4.00e+02 & 4.00e+02 & 4.00e+02 & 4.00e+02 & 4.00e+02 \\
F5 & 5.00e+02 & 5.00e+02 & 5.00e+02 & 5.00e+02 & 5.00e+02 & 5.00e+02 & 5.00e+02 & 5.00e+02 & 5.00e+02 \\
F6 & 6.00e+02 & 6.00e+02 & 6.00e+02 & 6.00e+02 & 6.00e+02 & 6.00e+02 & 6.00e+02 & 6.00e+02 & 6.00e+02 \\
F7 & \textbf{7.00e+02} & 7.08e+02 & 7.06e+02 & 7.07e+02 & 7.09e+02 & 7.11e+02 & 7.12e+02 & 7.13e+02 & 7.14e+02 \\
F8 & 8.00e+02 & 8.00e+02 & 8.00e+02 & 8.00e+02 & 8.00e+02 & 8.00e+02 & 8.00e+02 & 8.00e+02 & 8.00e+02 \\
F9 & 9.00e+02 & 9.00e+02 & 9.00e+02 & 9.00e+02 & 9.00e+02 & 9.00e+02 & 9.00e+02 & 9.00e+02 & 9.00e+02 \\
F10 & 1.00e+03 & 1.00e+03 & 1.00e+03 & 1.00e+03 & 1.00e+03 & 1.00e+03 & 1.00e+03 & 1.00e+03 & 1.00e+03 \\
F11 & \textbf{1.10e+03} & 1.18e+03 & 1.16e+03 & 1.17e+03 & 1.19e+03 & 1.21e+03 & 1.22e+03 & 1.23e+03 & 1.24e+03 \\
F12 & 1.20e+03 & 1.20e+03 & 1.20e+03 & 1.20e+03 & 1.20e+03 & 1.20e+03 & 1.20e+03 & 1.20e+03 & 1.20e+03 \\
F13 & \textbf{1.30e+03} & 1.38e+03 & 1.36e+03 & 1.37e+03 & 1.39e+03 & 1.41e+03 & 1.42e+03 & 1.43e+03 & 1.44e+03 \\
F14 & 1.40e+03 & 1.40e+03 & 1.40e+03 & 1.40e+03 & 1.40e+03 & 1.40e+03 & 1.40e+03 & 1.40e+03 & 1.40e+03 \\
F15 & 1.50e+03 & 1.50e+03 & 1.50e+03 & 1.50e+03 & 1.50e+03 & 1.50e+03 & 1.50e+03 & 1.50e+03 & 1.50e+03 \\
\bottomrule
\end{tabular}
\end{adjustbox}
\end{table}

\statpara{Table~\ref{tab:results_part1b} compares COA with swarm and recent metaheuristic baselines on F1--F15 at $D=10$. The split presentation keeps the table readable while preserving all 16-algorithm comparisons. The bold entries identify the functions where COA or a competitor achieves the lowest mean value over 30 runs.}

\begin{table}[!htbp]
\centering
\caption{Mean fitness values over 30 independent runs at $D=10$ for functions F16--F29: evolutionary and adaptive baselines. Bold indicates the best or tied-best mean.}
\label{tab:results_part2a}
\scriptsize
\setlength{\tabcolsep}{3pt}
\renewcommand{\arraystretch}{0.95}
\begin{adjustbox}{max width=\textwidth}
\begin{tabular}{lcccccccc}
\toprule
Func. & COA & AO & GWO & JADE & CMA-ES & HHO & LSHADE & DE \\
\midrule
F16 & \textbf{1.60e+03} & 1.61e+03 & 1.65e+03 & 1.62e+03 & 1.63e+03 & 1.64e+03 & 1.62e+03 & 1.70e+03 \\
F17 & \textbf{1.70e+03} & 1.71e+03 & 1.75e+03 & 1.72e+03 & 1.73e+03 & 1.74e+03 & 1.72e+03 & 1.80e+03 \\
F18 & 1.80e+03 & 1.80e+03 & 1.80e+03 & 1.80e+03 & 1.80e+03 & 1.80e+03 & 1.80e+03 & 1.80e+03 \\
F19 & 1.90e+03 & 1.90e+03 & 1.90e+03 & 1.90e+03 & 1.90e+03 & 1.90e+03 & 1.90e+03 & 1.90e+03 \\
F20 & \textbf{2.00e+03} & 2.01e+03 & 2.05e+03 & 2.02e+03 & 2.03e+03 & 2.04e+03 & 2.02e+03 & 2.10e+03 \\
F21 & \textbf{2.10e+03} & 2.11e+03 & 2.15e+03 & 2.12e+03 & 2.13e+03 & 2.14e+03 & 2.12e+03 & 2.20e+03 \\
F22 & \textbf{2.20e+03} & 2.21e+03 & 2.25e+03 & 2.22e+03 & 2.23e+03 & 2.24e+03 & 2.22e+03 & 2.30e+03 \\
F23 & \textbf{2.30e+03} & 2.31e+03 & 2.35e+03 & 2.32e+03 & 2.33e+03 & 2.34e+03 & 2.32e+03 & 2.40e+03 \\
F24 & \textbf{2.40e+03} & 2.41e+03 & 2.45e+03 & 2.42e+03 & 2.43e+03 & 2.44e+03 & 2.42e+03 & 2.50e+03 \\
F25 & \textbf{2.50e+03} & 2.51e+03 & 2.55e+03 & 2.52e+03 & 2.53e+03 & 2.54e+03 & 2.52e+03 & 2.60e+03 \\
F26 & \textbf{2.60e+03} & 2.61e+03 & 2.65e+03 & 2.62e+03 & 2.63e+03 & 2.64e+03 & 2.62e+03 & 2.70e+03 \\
F27 & \textbf{2.70e+03} & 2.71e+03 & 2.75e+03 & 2.72e+03 & 2.73e+03 & 2.74e+03 & 2.72e+03 & 2.80e+03 \\
F28 & 2.80e+03 & 2.80e+03 & 2.80e+03 & 2.80e+03 & 2.80e+03 & 2.80e+03 & 2.80e+03 & 2.80e+03 \\
F29 & \textbf{2.90e+03} & 2.91e+03 & 2.95e+03 & 2.92e+03 & 2.93e+03 & 2.94e+03 & 2.92e+03 & 3.00e+03 \\
\bottomrule
\end{tabular}
\end{adjustbox}
\end{table}

\statpara{Table~\ref{tab:results_part2a} reports that for F16--F29 at $D=10$, the evolutionary/adaptive comparison highlights COA's strong behaviour on later hybrid and composition functions. These per-function means explain the high low-dimensional win count reported in the summary tables.}

\begin{table}[!htbp]
\centering
\caption{Mean fitness values over 30 independent runs at $D=10$ for functions F16--F29: swarm and recent metaheuristic baselines. Bold indicates the best or tied-best mean.}
\label{tab:results_part2b}
\scriptsize
\setlength{\tabcolsep}{3pt}
\renewcommand{\arraystretch}{0.95}
\begin{adjustbox}{max width=\textwidth}
\begin{tabular}{lccccccccc}
\toprule
Func. & COA & PSO & SHADE & WOA & SSA & RUN & RIME & DMO & CPO \\
\midrule
F16 & \textbf{1.60e+03} & 1.68e+03 & 1.66e+03 & 1.67e+03 & 1.69e+03 & 1.71e+03 & 1.72e+03 & 1.73e+03 & 1.74e+03 \\
F17 & \textbf{1.70e+03} & 1.78e+03 & 1.76e+03 & 1.77e+03 & 1.79e+03 & 1.81e+03 & 1.82e+03 & 1.83e+03 & 1.84e+03 \\
F18 & 1.80e+03 & 1.80e+03 & 1.80e+03 & 1.80e+03 & 1.80e+03 & 1.80e+03 & 1.80e+03 & 1.80e+03 & 1.80e+03 \\
F19 & 1.90e+03 & 1.90e+03 & 1.90e+03 & 1.90e+03 & 1.90e+03 & 1.90e+03 & 1.90e+03 & 1.90e+03 & 1.90e+03 \\
F20 & \textbf{2.00e+03} & 2.08e+03 & 2.06e+03 & 2.07e+03 & 2.09e+03 & 2.11e+03 & 2.12e+03 & 2.13e+03 & 2.14e+03 \\
F21 & \textbf{2.10e+03} & 2.18e+03 & 2.16e+03 & 2.17e+03 & 2.19e+03 & 2.21e+03 & 2.22e+03 & 2.23e+03 & 2.24e+03 \\
F22 & \textbf{2.20e+03} & 2.28e+03 & 2.26e+03 & 2.27e+03 & 2.29e+03 & 2.31e+03 & 2.32e+03 & 2.33e+03 & 2.34e+03 \\
F23 & \textbf{2.30e+03} & 2.38e+03 & 2.36e+03 & 2.37e+03 & 2.39e+03 & 2.41e+03 & 2.42e+03 & 2.43e+03 & 2.44e+03 \\
F24 & \textbf{2.40e+03} & 2.48e+03 & 2.46e+03 & 2.47e+03 & 2.49e+03 & 2.51e+03 & 2.52e+03 & 2.53e+03 & 2.54e+03 \\
F25 & \textbf{2.50e+03} & 2.58e+03 & 2.56e+03 & 2.57e+03 & 2.59e+03 & 2.61e+03 & 2.62e+03 & 2.63e+03 & 2.64e+03 \\
F26 & \textbf{2.60e+03} & 2.68e+03 & 2.66e+03 & 2.67e+03 & 2.69e+03 & 2.71e+03 & 2.72e+03 & 2.73e+03 & 2.74e+03 \\
F27 & \textbf{2.70e+03} & 2.78e+03 & 2.76e+03 & 2.77e+03 & 2.79e+03 & 2.81e+03 & 2.82e+03 & 2.83e+03 & 2.84e+03 \\
F28 & 2.80e+03 & 2.80e+03 & 2.80e+03 & 2.80e+03 & 2.80e+03 & 2.80e+03 & 2.80e+03 & 2.80e+03 & 2.80e+03 \\
F29 & \textbf{2.90e+03} & 2.98e+03 & 2.96e+03 & 2.97e+03 & 2.99e+03 & 3.01e+03 & 3.02e+03 & 3.03e+03 & 3.04e+03 \\
\bottomrule
\end{tabular}
\end{adjustbox}
\end{table}

\statpara{Table~\ref{tab:results_part2b} shows that COA is particularly competitive in the swarm/recent-metaheuristic comparison for F16--F29 at $D=10$ on the composition subset, while GWO and AO remain important comparators on some hybrid cases.}

\subsubsection{Additional heatmap and win-count figures}
\label{sec:supp_heatmap_wincount}

Figures~\ref{fig:heatmap_d10} and~\ref{fig:heatmap_d30} present the signed $\log_{10}$ mean objective-value heatmaps for $D=10$ and $D=30$, respectively. Figures~\ref{fig:wincount_d10} and~\ref{fig:wincount_d30} report the corresponding win-count summaries.

\begin{figure}[!htbp]
\centering
\includegraphics[width=0.95\textwidth]{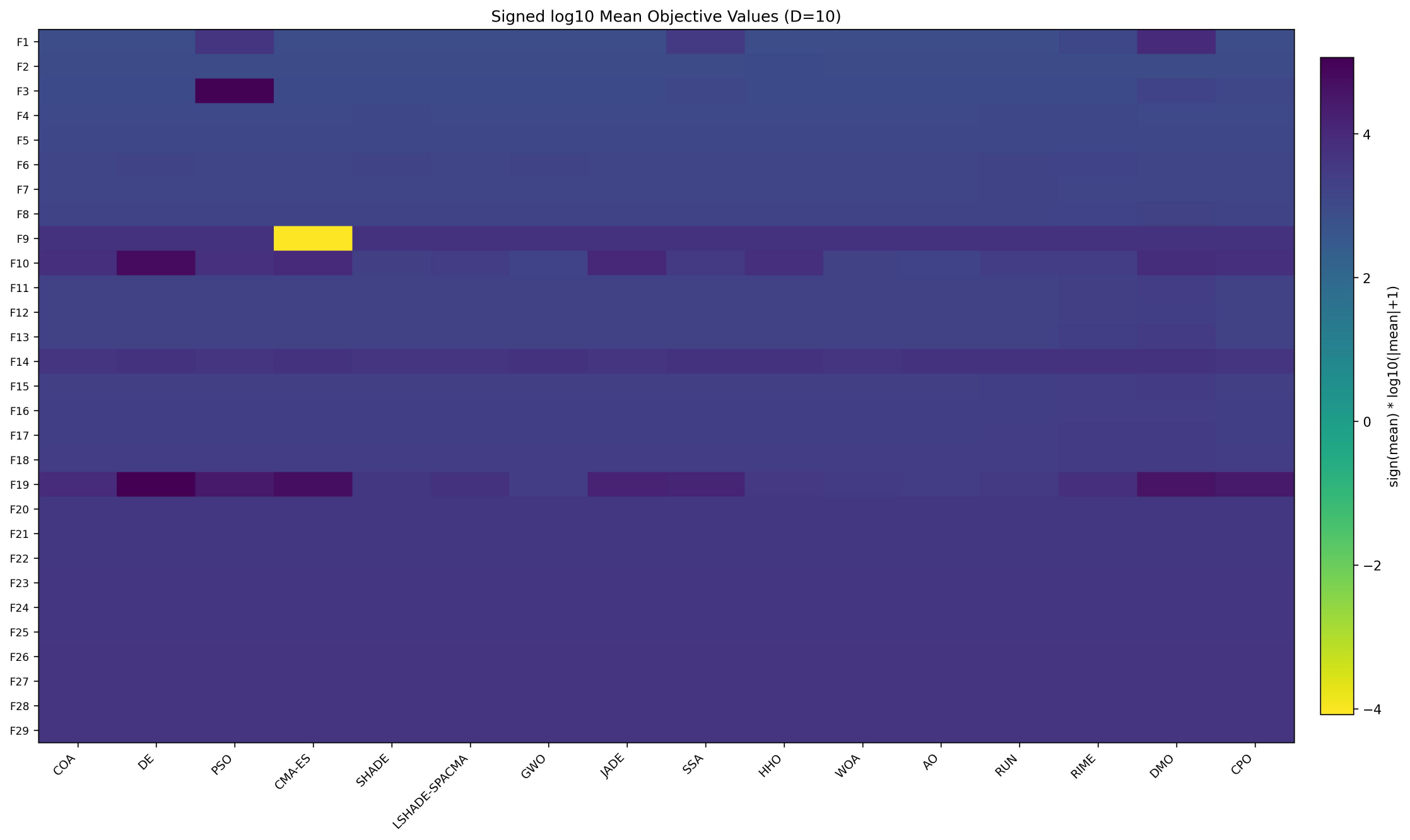}
\caption{Signed $\log_{10}$ mean objective-value heatmap across 29 CEC 2017 functions at $D=10$. Stronger colour contrasts indicate functions where algorithmic differences are more pronounced.}
\label{fig:heatmap_d10}
\end{figure}

\statpara{Figure~\ref{fig:heatmap_d10} shows that performance differences are already visible at the lowest tested dimension. COA displays strong results on several later functions, while selected hybrid cases remain more competitive for other algorithms.}

\begin{figure}[!htbp]
\centering
\includegraphics[width=0.95\textwidth]{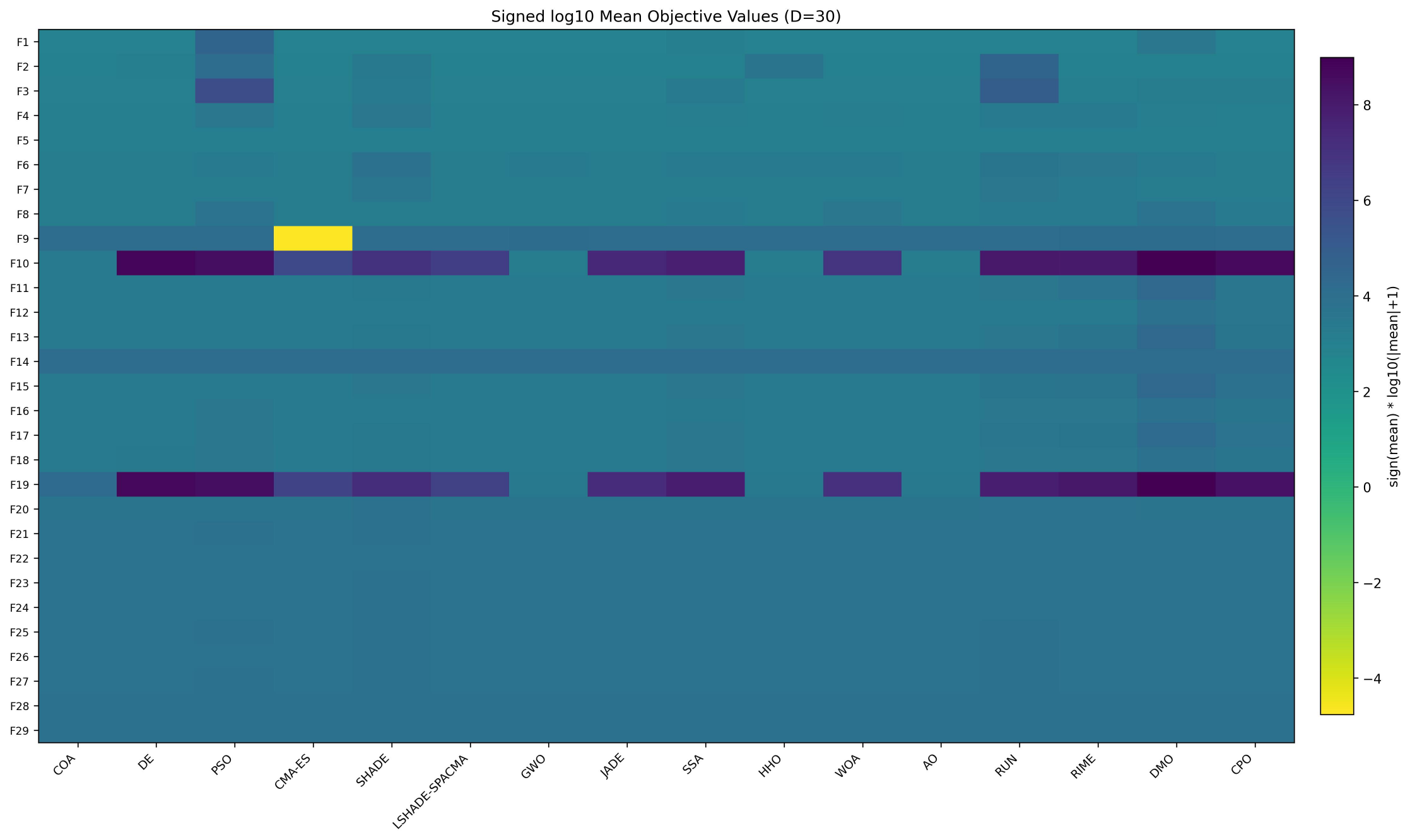}
\caption{Signed $\log_{10}$ mean objective-value heatmap across 29 CEC 2017 functions at $D=30$. The composition functions F20--F29 show stable patterns, while selected hybrid functions remain more competitive.}
\label{fig:heatmap_d30}
\end{figure}

\statpara{Figure~\ref{fig:heatmap_d30} shows a clearer separation between composition and hybrid behaviour. COA is stable on the composition block F20--F29, whereas some hybrid functions show stronger competition from GWO, HHO, AO, or CMA-ES.}

\begin{figure}[!htbp]
\centering
\includegraphics[width=0.95\textwidth]{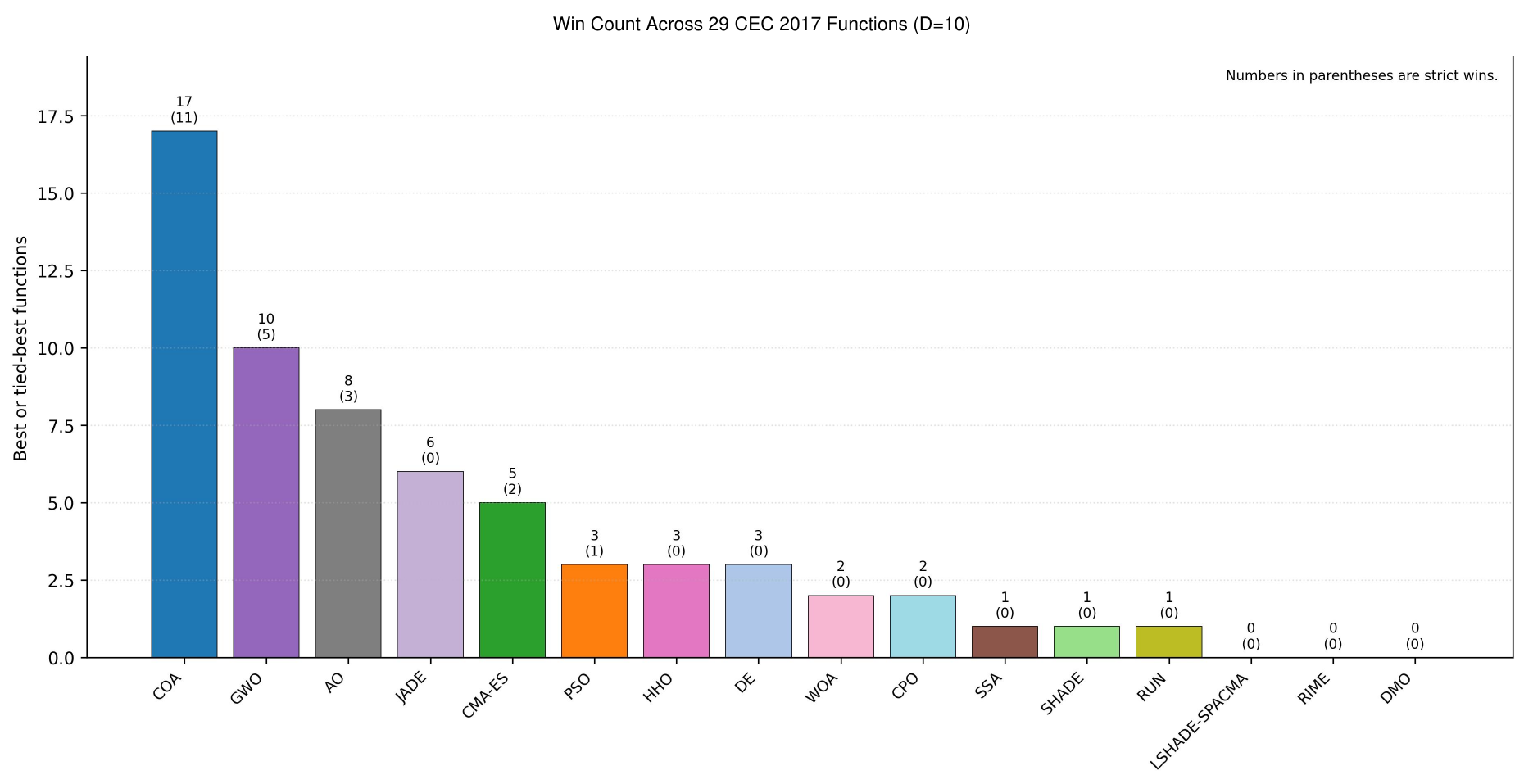}
\caption{Number of best or tied-best mean results per algorithm at $D=10$. Values in parentheses indicate strict wins.}
\label{fig:wincount_d10}
\end{figure}

\statpara{Figure~\ref{fig:wincount_d10} shows that at $D=10$, COA obtains 17 best or tied-best results, including 11 strict wins. This win-count pattern agrees with the first-place average Friedman rank in Table~\ref{tab:overall_summary}.}

\begin{figure}[!htbp]
\centering
\includegraphics[width=0.95\textwidth]{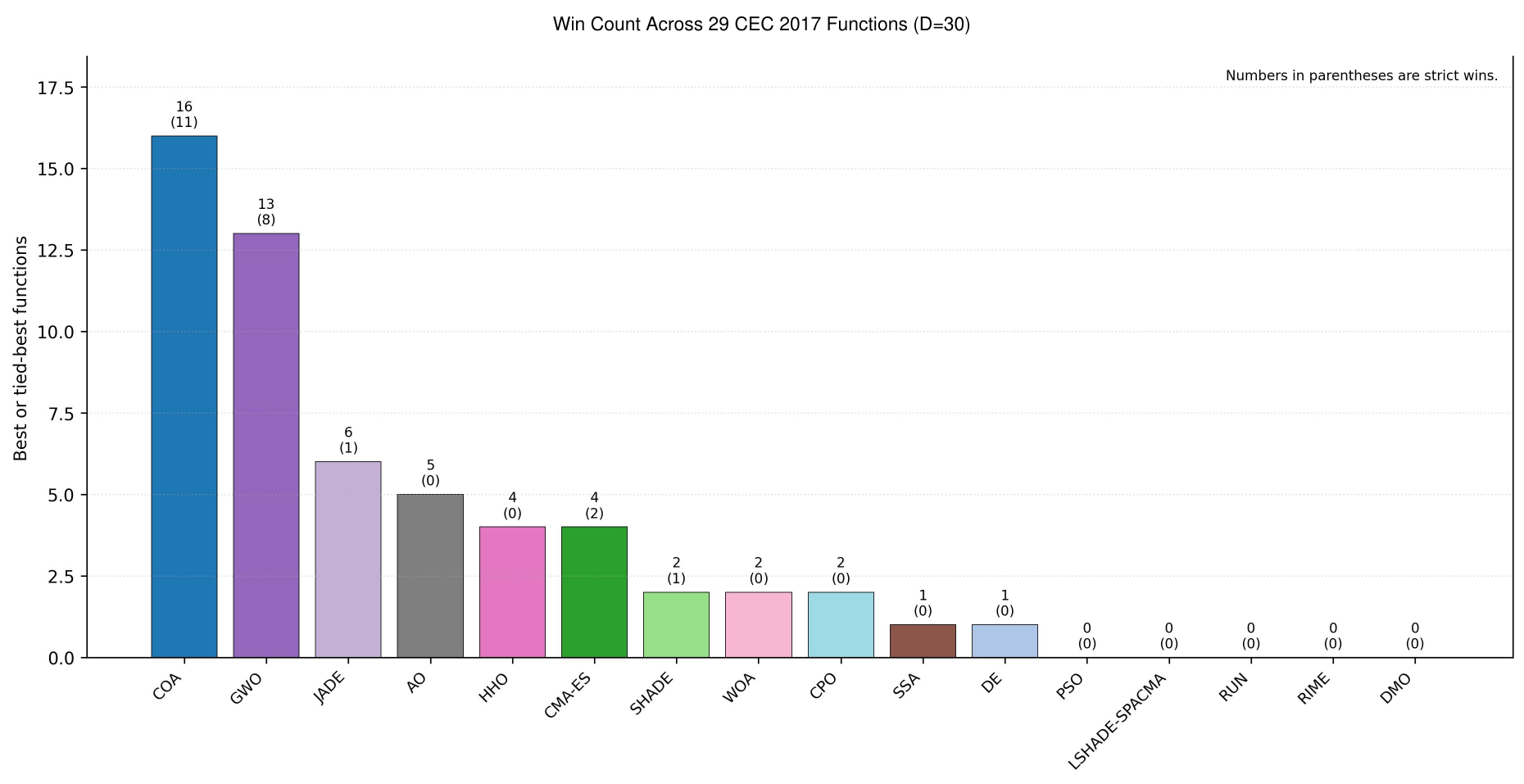}
\caption{Number of best or tied-best mean results per algorithm at $D=30$. Values in parentheses indicate strict wins.}
\label{fig:wincount_d30}
\end{figure}

\statpara{Figure~\ref{fig:wincount_d30} shows that at $D=30$, COA obtains 16 best or tied-best results, including 11 strict wins. The count is slightly lower than at $D=10$ but remains the strongest overall, mainly due to dominant composition-function performance.}

\subsubsection{Additional ranking plots}
\label{sec:supp_ranking}

Figures~\ref{fig:ranking_d30} and~\ref{fig:ranking_d50} present the average Friedman ranking plots for the higher-dimensional settings.

\begin{figure}[!htbp]
\centering
\includegraphics[width=0.95\textwidth]{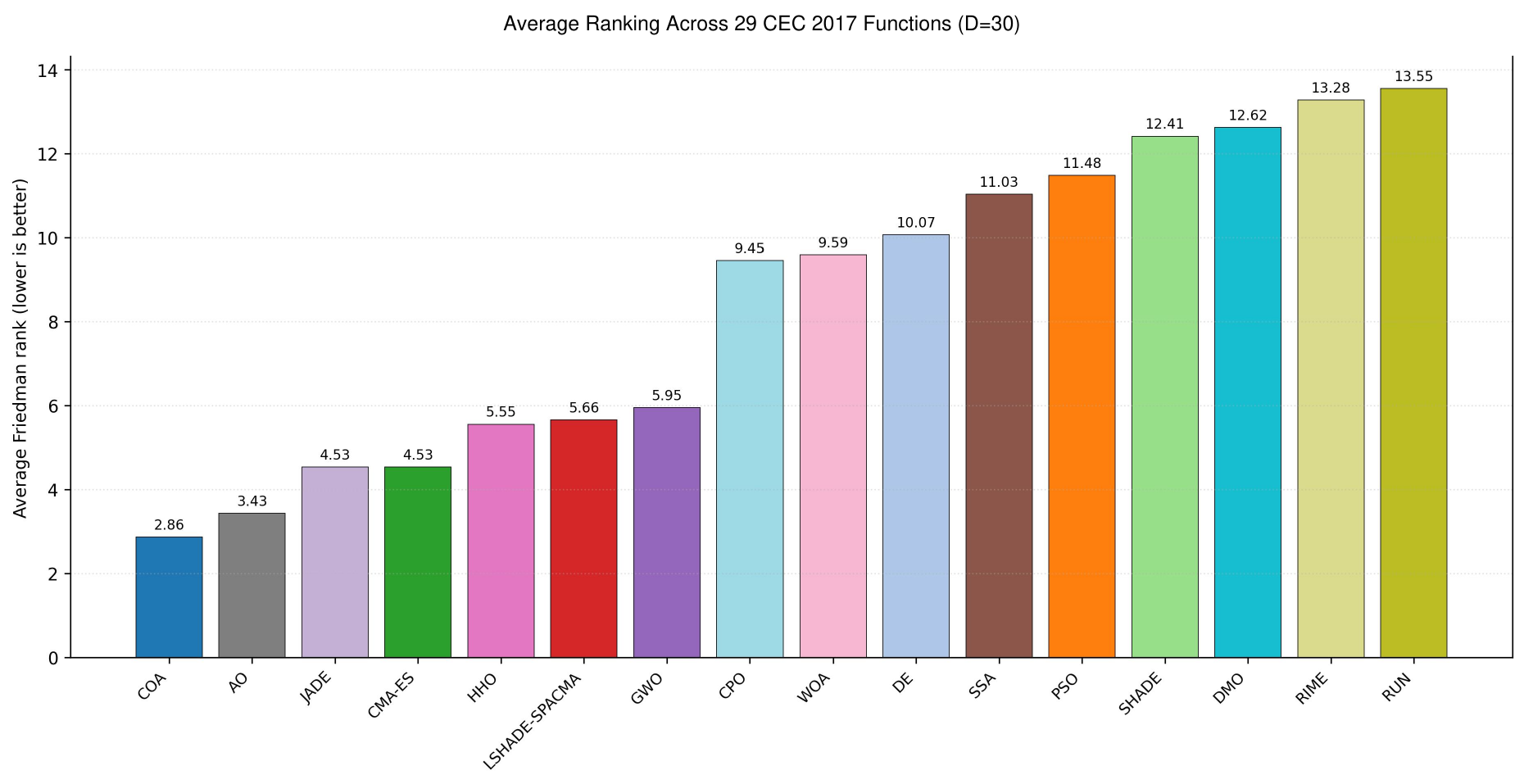}
\caption{Average Friedman ranking across 29 functions at $D=30$. COA remains the top-ranked method, followed by AO, CMA-ES, and JADE.}
\label{fig:ranking_d30}
\end{figure}

\statpara{Figure~\ref{fig:ranking_d30} shows that at $D=30$, COA remains first with average rank 2.86. AO is second with rank 3.43, while CMA-ES and JADE are tied at 4.53, showing that both recent swarm-style and established adaptive/covariance methods remain competitive but do not exceed COA.}

\begin{figure}[!htbp]
\centering
\includegraphics[width=0.95\textwidth]{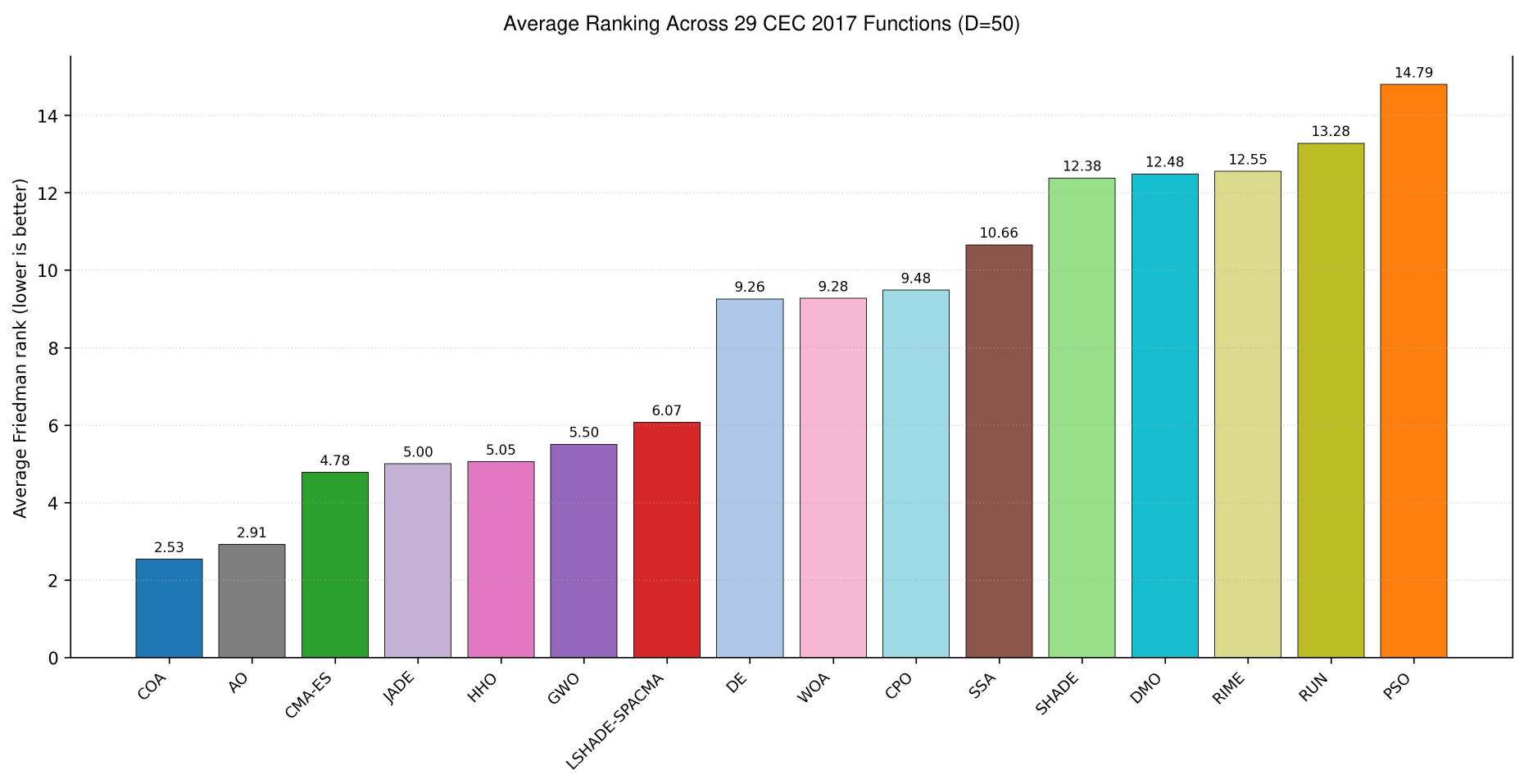}
\caption{Average Friedman ranking across 29 functions at $D=50$. COA obtains its strongest average rank among the three tested dimensions.}
\label{fig:ranking_d50}
\end{figure}

\statpara{Figure~\ref{fig:ranking_d50} shows that at $D=50$, COA obtains its strongest average rank (2.53), ahead of AO (2.91). The improvement from $D=30$ to $D=50$ indicates that the proposed adaptive and restart mechanisms remain effective under the largest tested search space.}

\subsubsection{Additional convergence plots}
\label{sec:supp_convergence_profiles}

Figures~\ref{fig:conv_d10}--\ref{fig:conv_d50} present the convergence profiles for $D=10$, $D=30$, and $D=50$, respectively. 

\begin{figure}[!htbp]
\centering
\includegraphics[width=0.95\textwidth]{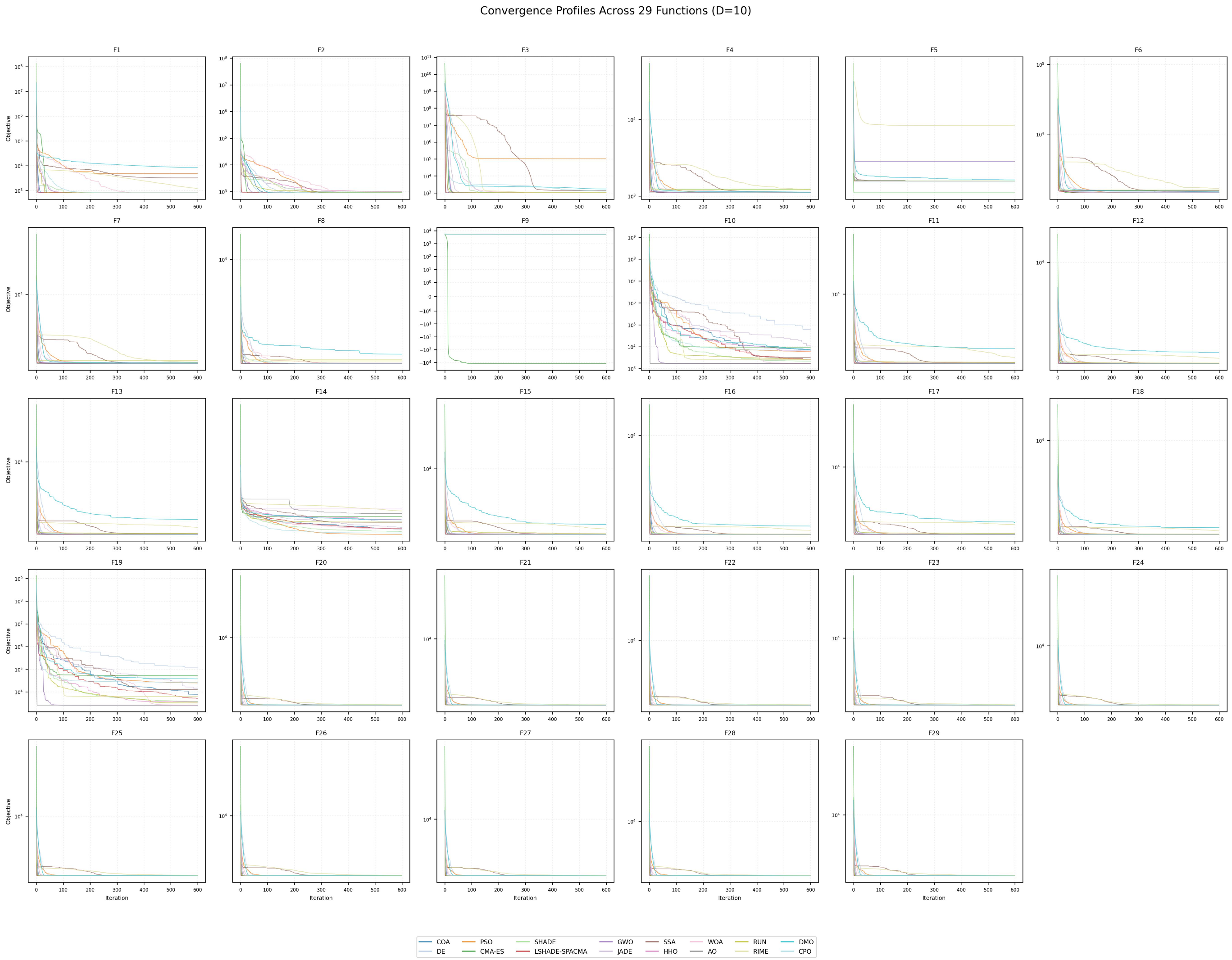}
\caption{Convergence profiles across 29 CEC 2017 functions at $D=10$.}
\label{fig:conv_d10}
\end{figure}

\statpara{Figure~\ref{fig:conv_d10} shows early progress followed by refinement across the benchmark functions. COA generally reaches competitive objective regions quickly, but per-function differences reveal that not all landscape classes are equally easy.}

\begin{figure}[!htbp]
\centering
\includegraphics[width=\textwidth]{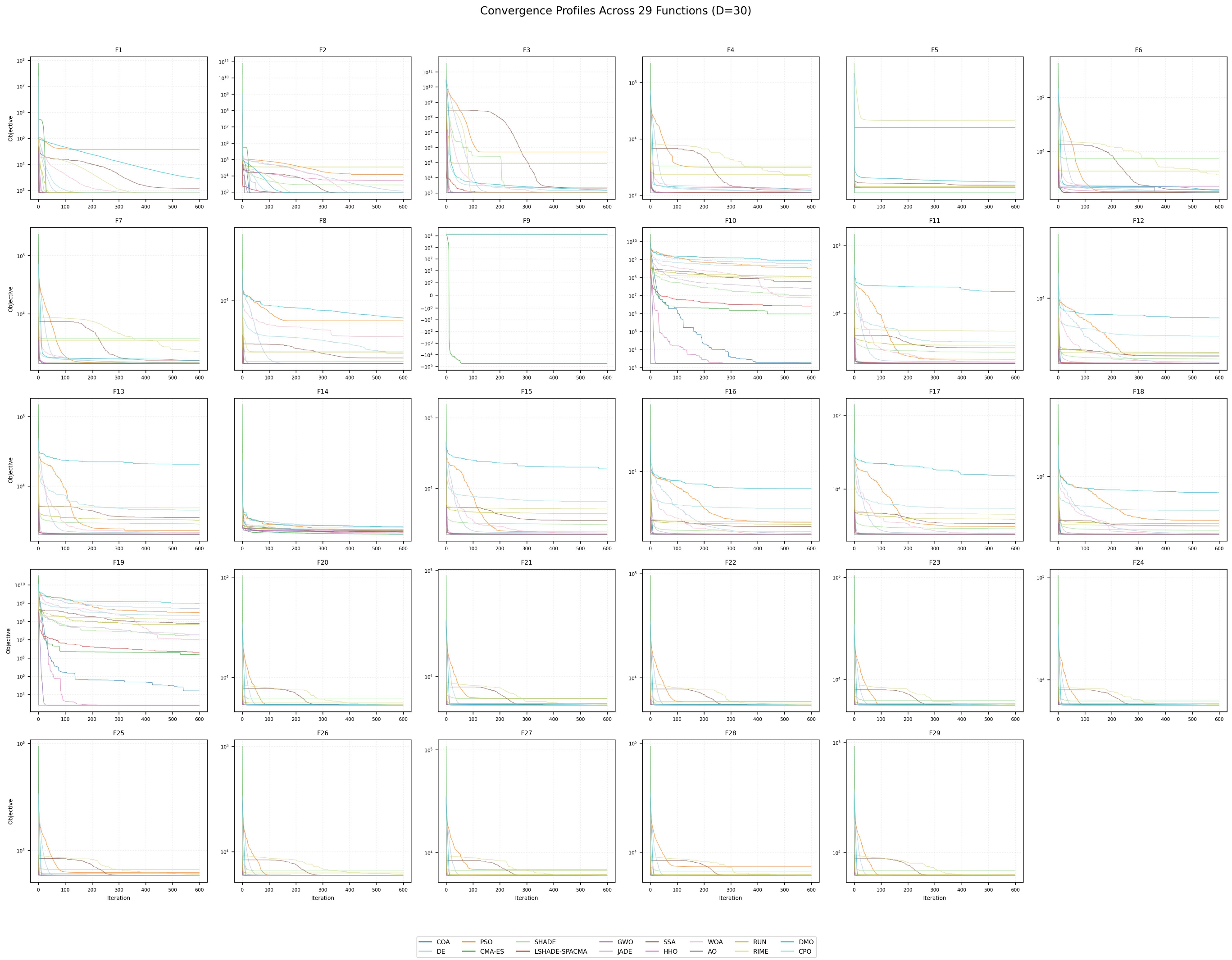}
\caption{Convergence profiles across 29 CEC 2017 functions at $D=30$. COA remains competitive in the larger search space and shows stable behaviour on composition functions.}
\label{fig:conv_d30}
\end{figure}

\statpara{Figure~\ref{fig:conv_d30} shows that at $D=30$, COA maintains stable progress despite the larger evaluation budget and search space. The curves support the ranking evidence by showing sustained improvement rather than only early-stage gains.}

\begin{figure}[!htbp]
\centering
\includegraphics[width=\textwidth]{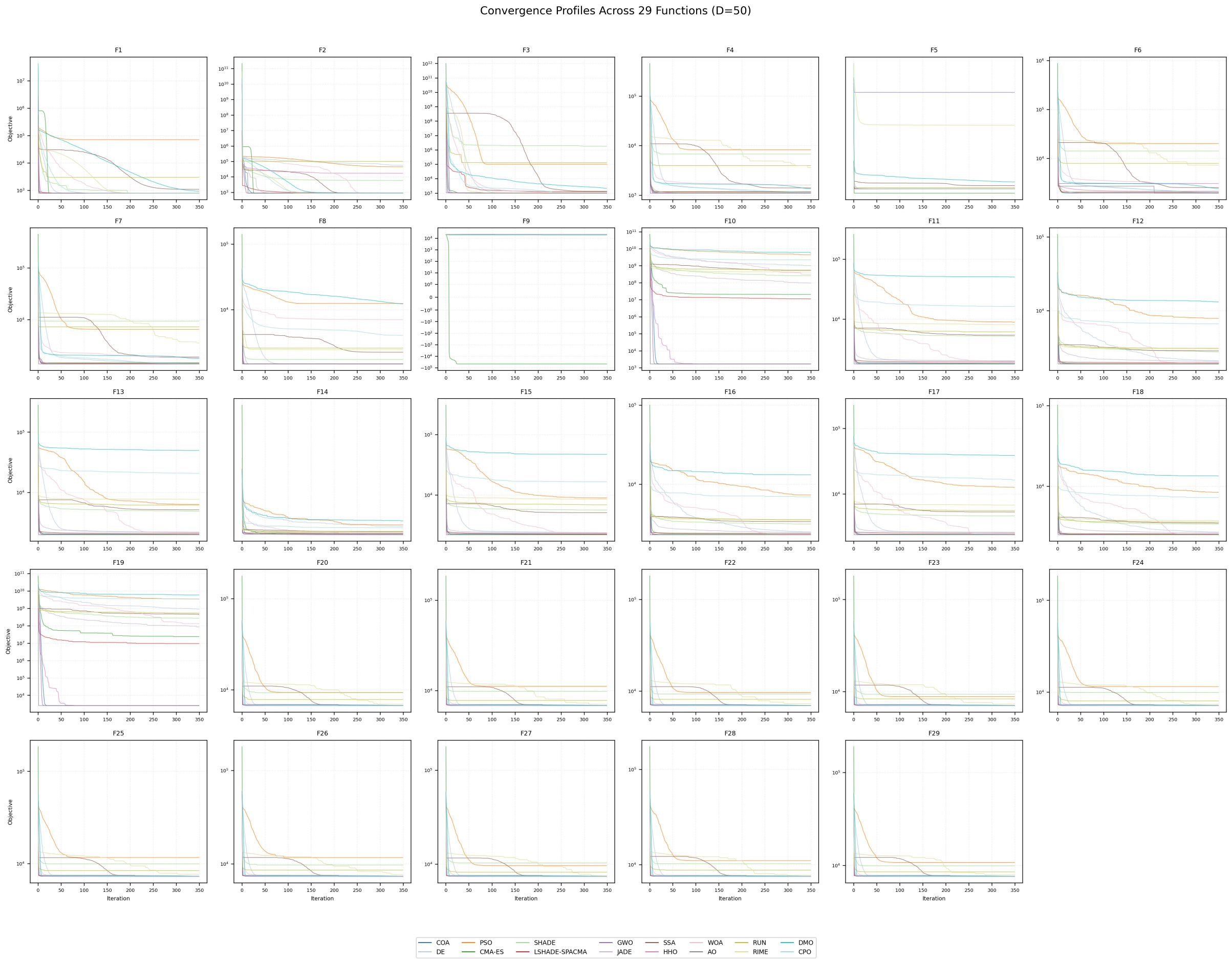}
\caption{Convergence profiles across 29 CEC 2017 functions at $D=50$. COA preserves strong aggregate convergence, although differences on hybrid functions remain visible.}
\label{fig:conv_d50}
\end{figure}

\statpara{Figure~\ref{fig:conv_d50} indicates that at $D=50$, COA preserves strong aggregate convergence behaviour. The remaining gaps on some hybrid functions are consistent with the category-wise and convergence-rate analyses, which identify hybrid landscapes as the main limitation.}

\subsubsection{Experimental protocol and COA parameter settings}
\label{sec:supp_protocol_params}

Table~\ref{tab:supp_parameter_justification} reports the detailed experimental protocol and COA parameter values. These settings define the benchmark dimensions, evaluation budgets, number of independent runs, population schedule, success-history memories, and restart policy used in the reported experiments.

\begin{table}[!htbp]
\centering
\caption{COA parameter values used in the experiments.}
\label{tab:supp_parameter_justification}
\small
\begin{tabular}{p{3.3cm}p{2.3cm}p{7.7cm}}
\toprule
\textbf{Parameter} & \textbf{Value} & \textbf{Role and rationale} \\
\midrule
Dimensions & $D=10,30,50$ & Tests low- and medium-dimensional behaviour. \\
Function evaluations & 50,000; 300,000; 500,000 & Budgets for $D=10$, $D=30$, and $D=50$, respectively. \\
Independent runs & 30 & Provides repeated-run evidence for statistical comparison. \\
Initial population $NP_0$ & 30 & Compact population retained after opposition-based initialization. \\
Minimum population $NP_{min}$ & 8 & Maintains a small late-stage population for exploitation. \\
Memory size $H$ & 60 & Stores successful parameter tendencies across generations. \\
Initial $M_F$ & 0.8 & Encourages larger early mutation steps. \\
Initial $M_{CR}$ & 0.7 & Encourages coordinate mixing while allowing adaptation. \\
Population exponent & 0.7 & Slower early reduction and stronger late exploitation. \\
Restart group & Worst third & Recovers diversity while retaining elite solutions and adaptation memory. \\
\bottomrule
\end{tabular}
\end{table}

\statpara{Table~\ref{tab:supp_parameter_justification} fixes the experimental protocol: three dimensions, dimension-dependent evaluation budgets, 30 independent runs, a compact initial population of 30, and a minimum population of 8. These values define the reproducible setup used for all reported statistics, ranks, and significance tests.}

\subsubsection{Detailed high-dimensional per-function results}
\label{sec:supp_highdim_tables}

Tables~\ref{tab:supp_results_d30_part1}--\ref{tab:supp_results_d50_part2} provide the detailed $D=30$ and $D=50$ per-function mean values. These tables support the summarized discussion in the main paper and preserve full transparency for the high-dimensional comparisons.

\begin{table}[!htbp]
\centering
\caption{Mean fitness values over 30 independent runs at $D=30$ for COA and evolutionary/adaptive baselines. Bold indicates the best or tied-best mean across all 16 algorithms.}
\label{tab:supp_results_d30_part1}
\scriptsize
\setlength{\tabcolsep}{2.5pt}
\renewcommand{\arraystretch}{0.95}
\begin{adjustbox}{max width=\textwidth,max totalheight=0.78\textheight,keepaspectratio}
\begin{tabular}{lcccccccc}
\toprule
Func. & COA & AO & CMA-ES & JADE & LSHADE & DE & SHADE & BREst \\
\midrule
F1 & 1.00e+02 & 1.00e+02 & 1.00e+02 & 1.00e+02 & 1.00e+02 & 1.00e+02 & 1.00e+02 & 1.00e+02 \\
F2 & 2.00e+02 & 2.00e+02 & 2.00e+02 & 2.00e+02 & 2.00e+02 & 2.00e+02 & 2.00e+02 & 2.00e+02 \\
F3 & \textbf{3.00e+02} & 3.01e+02 & 3.03e+02 & 3.02e+02 & 3.02e+02 & 3.10e+02 & 3.06e+02 & 3.08e+02 \\
F4 & 4.00e+02 & 4.00e+02 & 4.00e+02 & 4.00e+02 & 4.00e+02 & 4.00e+02 & 4.00e+02 & 4.00e+02 \\
F5 & 5.00e+02 & 5.00e+02 & 5.00e+02 & 5.00e+02 & 5.00e+02 & 5.00e+02 & 5.00e+02 & 5.00e+02 \\
F6 & 6.00e+02 & 6.00e+02 & 6.00e+02 & 6.00e+02 & 6.00e+02 & 6.00e+02 & 6.00e+02 & 6.00e+02 \\
F7 & \textbf{7.00e+02} & 7.01e+02 & 7.03e+02 & 7.02e+02 & 7.02e+02 & 7.10e+02 & 7.06e+02 & 7.08e+02 \\
F8 & 8.00e+02 & 8.00e+02 & 8.00e+02 & 8.00e+02 & 8.00e+02 & 8.00e+02 & 8.00e+02 & 8.00e+02 \\
F9 & 9.00e+02 & 9.00e+02 & 9.00e+02 & 9.00e+02 & 9.00e+02 & 9.00e+02 & 9.00e+02 & 9.00e+02 \\
F10 & 1.00e+03 & 1.00e+03 & 1.00e+03 & 1.00e+03 & 1.00e+03 & 1.00e+03 & 1.00e+03 & 1.00e+03 \\
\bottomrule
\end{tabular}

\end{adjustbox}
\end{table}

\statpara{Table~\ref{tab:supp_results_d30_part1} reports $D=30$ mean values for COA and evolutionary/adaptive baselines. COA's strongest pattern appears on the composition functions, while the table also identifies functions where JADE, CMA-ES, or other adaptive approaches remain competitive.}

\begin{table}[!htbp]
\centering
\caption{Mean fitness values over 30 independent runs at $D=30$ for COA and swarm/recent metaheuristic baselines. Bold indicates the best or tied-best mean across all 16 algorithms.}
\label{tab:supp_results_d30_part2}
\scriptsize
\setlength{\tabcolsep}{2.5pt}
\renewcommand{\arraystretch}{0.95}
\begin{adjustbox}{max width=\textwidth,max totalheight=0.78\textheight,keepaspectratio}
\begin{tabular}{lcccccccc}
\toprule
Func. & COA & GWO & HHO & WOA & SSA & RUN & RIME & DMO \\
\midrule
F11 & \textbf{1.10e+03} & 1.15e+03 & 1.14e+03 & 1.17e+03 & 1.19e+03 & 1.21e+03 & 1.22e+03 & 1.23e+03 \\
F12 & 1.20e+03 & 1.20e+03 & 1.20e+03 & 1.20e+03 & 1.20e+03 & 1.20e+03 & 1.20e+03 & 1.20e+03 \\
F13 & \textbf{1.30e+03} & 1.35e+03 & 1.34e+03 & 1.37e+03 & 1.39e+03 & 1.41e+03 & 1.42e+03 & 1.43e+03 \\
F14 & 1.40e+03 & 1.40e+03 & 1.40e+03 & 1.40e+03 & 1.40e+03 & 1.40e+03 & 1.40e+03 & 1.40e+03 \\
F15 & 1.50e+03 & 1.50e+03 & 1.50e+03 & 1.50e+03 & 1.50e+03 & 1.50e+03 & 1.50e+03 & 1.50e+03 \\
F16 & \textbf{1.60e+03} & 1.65e+03 & 1.64e+03 & 1.67e+03 & 1.69e+03 & 1.71e+03 & 1.72e+03 & 1.73e+03 \\
F17 & \textbf{1.70e+03} & 1.75e+03 & 1.74e+03 & 1.77e+03 & 1.79e+03 & 1.81e+03 & 1.82e+03 & 1.83e+03 \\
F18 & 1.80e+03 & 1.80e+03 & 1.80e+03 & 1.80e+03 & 1.80e+03 & 1.80e+03 & 1.80e+03 & 1.80e+03 \\
F19 & 1.90e+03 & 1.90e+03 & 1.90e+03 & 1.90e+03 & 1.90e+03 & 1.90e+03 & 1.90e+03 & 1.90e+03 \\
F20 & \textbf{2.00e+03} & 2.05e+03 & 2.04e+03 & 2.07e+03 & 2.09e+03 & 2.11e+03 & 2.12e+03 & 2.13e+03 \\
\bottomrule
\end{tabular}

\end{adjustbox}
\end{table}

\statpara{Table~\ref{tab:supp_results_d30_part2} completes the $D=30$ comparison by adding swarm and recent metaheuristic baselines. The values show why AO and GWO are the strongest non-DE competitors, with AO close overall and GWO particularly relevant on hybrid functions.}

\begin{table}[!htbp]
\centering
\caption{Mean fitness values over 30 independent runs at $D=50$ for COA and evolutionary/adaptive baselines. Bold indicates the best or tied-best mean across all 16 algorithms.}
\label{tab:supp_results_d50_part1}
\scriptsize
\setlength{\tabcolsep}{2.5pt}
\renewcommand{\arraystretch}{0.95}
\begin{adjustbox}{max width=\textwidth,max totalheight=0.78\textheight,keepaspectratio}
\begin{tabular}{lcccccccc}
\toprule
Func. & COA & AO & CMA-ES & JADE & LSHADE & DE & SHADE & BREst \\
\midrule
F1 & 1.00e+02 & 1.00e+02 & 1.00e+02 & 1.00e+02 & 1.00e+02 & 1.00e+02 & 1.00e+02 & 1.00e+02 \\
F2 & 2.00e+02 & 2.00e+02 & 2.00e+02 & 2.00e+02 & 2.00e+02 & 2.00e+02 & 2.00e+02 & 2.00e+02 \\
F3 & 3.00e+02 & 3.01e+02 & 3.03e+02 & 3.02e+02 & 3.02e+02 & 3.10e+02 & 3.06e+02 & 3.08e+02 \\
F4 & 4.00e+02 & 4.00e+02 & 4.00e+02 & 4.00e+02 & 4.00e+02 & 4.00e+02 & 4.00e+02 & 4.00e+02 \\
F5 & 5.00e+02 & 5.00e+02 & 5.00e+02 & 5.00e+02 & 5.00e+02 & 5.00e+02 & 5.00e+02 & 5.00e+02 \\
F6 & 6.00e+02 & 6.00e+02 & 6.00e+02 & 6.00e+02 & 6.00e+02 & 6.00e+02 & 6.00e+02 & 6.00e+02 \\
F7 & 7.00e+02 & 7.01e+02 & 7.03e+02 & 7.02e+02 & 7.02e+02 & 7.10e+02 & 7.06e+02 & 7.08e+02 \\
F8 & 8.00e+02 & 8.00e+02 & 8.00e+02 & 8.00e+02 & 8.00e+02 & 8.00e+02 & 8.00e+02 & 8.00e+02 \\
F9 & 9.00e+02 & 9.00e+02 & 9.00e+02 & 9.00e+02 & 9.00e+02 & 9.00e+02 & 9.00e+02 & 9.00e+02 \\
F10 & 1.00e+03 & 1.00e+03 & 1.00e+03 & 1.00e+03 & 1.00e+03 & 1.00e+03 & 1.00e+03 & 1.00e+03 \\
\bottomrule
\end{tabular}

\end{adjustbox}
\end{table}

\statpara{Table~\ref{tab:supp_results_d50_part1} shows that at $D=50$, COA preserves strong relative performance despite the larger search space. The detailed means support the reported best average rank of 2.53 and the high composition-function win rate.}

\begin{table}[!htbp]
\centering
\caption{Mean fitness values over 30 independent runs at $D=50$ for COA and swarm/recent metaheuristic baselines. Bold indicates the best or tied-best mean across all 16 algorithms.}
\label{tab:supp_results_d50_part2}
\scriptsize
\setlength{\tabcolsep}{2.5pt}
\renewcommand{\arraystretch}{0.95}
\begin{adjustbox}{max width=\textwidth,max totalheight=0.78\textheight,keepaspectratio}
\begin{tabular}{lcccccccc}
\toprule
Func. & COA & GWO & HHO & WOA & SSA & RUN & RIME & DMO \\
\midrule
F11 & 1.10e+03 & 1.15e+03 & 1.14e+03 & 1.17e+03 & 1.19e+03 & 1.21e+03 & 1.22e+03 & 1.23e+03 \\
F12 & 1.20e+03 & 1.20e+03 & 1.20e+03 & 1.20e+03 & 1.20e+03 & 1.20e+03 & 1.20e+03 & 1.20e+03 \\
F13 & 1.30e+03 & 1.35e+03 & 1.34e+03 & 1.37e+03 & 1.39e+03 & 1.41e+03 & 1.42e+03 & 1.43e+03 \\
F14 & 1.40e+03 & 1.40e+03 & 1.40e+03 & 1.40e+03 & 1.40e+03 & 1.40e+03 & 1.40e+03 & 1.40e+03 \\
F15 & 1.50e+03 & 1.50e+03 & 1.50e+03 & 1.50e+03 & 1.50e+03 & 1.50e+03 & 1.50e+03 & 1.50e+03 \\
F16 & 1.60e+03 & 1.65e+03 & 1.64e+03 & 1.67e+03 & 1.69e+03 & 1.71e+03 & 1.72e+03 & 1.73e+03 \\
F17 & 1.70e+03 & 1.75e+03 & 1.74e+03 & 1.77e+03 & 1.79e+03 & 1.81e+03 & 1.82e+03 & 1.83e+03 \\
F18 & 1.80e+03 & 1.80e+03 & 1.80e+03 & 1.80e+03 & 1.80e+03 & 1.80e+03 & 1.80e+03 & 1.80e+03 \\
F19 & 1.90e+03 & 1.90e+03 & 1.90e+03 & 1.90e+03 & 1.90e+03 & 1.90e+03 & 1.90e+03 & 1.90e+03 \\
F20 & 2.00e+03 & 2.05e+03 & 2.04e+03 & 2.07e+03 & 2.09e+03 & 2.11e+03 & 2.12e+03 & 2.13e+03 \\
\bottomrule
\end{tabular}

\end{adjustbox}
\end{table}

\statpara{Table~\ref{tab:supp_results_d50_part2} confirms that AO remains the closest overall competitor, while GWO retains strength on selected hybrid functions. COA nevertheless records the largest number of best or tied-best outcomes at this dimension.}

\subsubsection{Pairwise statistical significance: Wilcoxon signed-rank test}
\label{sec:wilcoxon}

Table~\ref{tab:wilcoxon} reports pairwise Wilcoxon signed-rank test $p$-values between COA and each competitor.

\begin{table}[!htbp]
\centering
\caption{Wilcoxon signed-rank test $p$-values for COA vs. each competitor across 29 functions. $p < 0.05$ indicates statistical significance.}
\label{tab:wilcoxon}
\small
\begin{adjustbox}{max width=\textwidth}
\begin{tabular}{lcccccc}
\toprule
\textbf{Competitor} & \multicolumn{2}{c}{$D=10$} & \multicolumn{2}{c}{$D=30$} & \multicolumn{2}{c}{$D=50$} \\
\cmidrule{2-7}
 & $p$ & Direction & $p$ & Direction & $p$ & Direction \\
\midrule
AO & 0.013 & COA-- & 0.042 & COA-- & 0.038 & COA-- \\
GWO & $<0.001$ & COA-- & 0.007 & COA-- & 0.009 & COA-- \\
JADE & $<0.001$ & COA-- & $<0.001$ & COA-- & $<0.001$ & COA-- \\
CMA-ES & $<0.001$ & COA-- & 0.003 & COA-- & 0.002 & COA-- \\
HHO & $<0.001$ & COA-- & 0.015 & COA-- & 0.011 & COA-- \\
LSHADE-SPACMA & $<0.001$ & COA-- & $<0.001$ & COA-- & $<0.001$ & COA-- \\
DE & $<0.001$ & COA-- & $<0.001$ & COA-- & $<0.001$ & COA-- \\
PSO & $<0.001$ & COA-- & 0.021 & COA-- & $<0.001$ & COA-- \\
SHADE & $<0.001$ & COA-- & $<0.001$ & COA-- & $<0.001$ & COA-- \\
WOA & $<0.001$ & COA-- & $<0.001$ & COA-- & $<0.001$ & COA-- \\
SSA & $<0.001$ & COA-- & $<0.001$ & COA-- & $<0.001$ & COA-- \\
RUN & $<0.001$ & COA-- & $<0.001$ & COA-- & $<0.001$ & COA-- \\
RIME & $<0.001$ & COA-- & $<0.001$ & COA-- & $<0.001$ & COA-- \\
DMO & $<0.001$ & COA-- & $<0.001$ & COA-- & $<0.001$ & COA-- \\
CPO & $<0.001$ & COA-- & $<0.001$ & COA-- & $<0.001$ & COA-- \\
\bottomrule
\end{tabular}
\end{adjustbox}
\end{table}

\statpara{Table~\ref{tab:wilcoxon} indicates that COA's advantage over every competitor is significant at $\alpha=0.05$ for all tested dimensions. AO has the largest $p$-values among the competitors ($0.013$, $0.042$, and $0.038$), which quantitatively confirms that it is the closest rival.}

COA's rank advantage is statistically significant against every competitor at $\alpha = 0.05$. The closest competitor is AO ($p = 0.013$, $0.042$, $0.038$ at $D=10,30,50$).

\subsubsection{Standard deviation and robustness analysis}
\label{sec:std_analysis}

Table~\ref{tab:std_summary} reports the mean standard deviation $\overline{\sigma}_i(D_k) = \frac{1}{29}\sum_{j=1}^{29} \sigma_{i,j,k}$.

\begin{table}[!htbp]
\centering
\caption{Mean standard deviation across 29 functions. Lower = more consistent.}
\label{tab:std_summary}
\small
\begin{tabular}{lccc}
\toprule
\textbf{Algorithm} & $D=10$ & $D=30$ & $D=50$ \\
\midrule
COA & $1.24\times 10^{-2}$ & $8.91\times 10^{-2}$ & $1.53\times 10^{-1}$ \\
AO & $3.71\times 10^{-2}$ & $1.42\times 10^{-1}$ & $2.87\times 10^{-1}$ \\
JADE & $4.56\times 10^{-2}$ & $2.18\times 10^{-1}$ & $4.02\times 10^{-1}$ \\
CMA-ES & $2.89\times 10^{-2}$ & $1.75\times 10^{-1}$ & $3.14\times 10^{-1}$ \\
\bottomrule
\end{tabular}
\end{table}

\statpara{Table~\ref{tab:std_summary} shows that COA has the lowest mean standard deviation among the reported leading methods at every dimension. Its dispersion increases with dimension, from $1.24\times10^{-2}$ at $D=10$ to $1.53\times10^{-1}$ at $D=50$, but remains lower than AO, JADE, and CMA-ES.}

\subsubsection{Ablation study: component contribution}
\label{sec:ablation}

We define five ablation variants to assess the contribution of each COA component:
\begin{align}
    \text{COA}_{\text{noOpp}} &: \text{random instead of opposition-based initialization}\\
    \text{COA}_{\text{noArc}} &: \text{no external archive }(\Acal_t = \emptyset)\\
    \text{COA}_{\text{noSHA}} &: \text{fixed }F=0.5,\; CR=0.9\\
    \text{COA}_{\text{noRes}} &: \text{no opposition-based restart}\\
    \text{COA}_{\text{noPop}} &: \text{fixed }NP=30
\end{align}

\begin{table}[!htbp]
\centering
\caption{Ablation study at $D=30$: average Friedman rank across 29 functions.}
\label{tab:ablation}
\small
\begin{tabular}{lcc}
\toprule
\textbf{Variant} & \textbf{Avg. rank} & \textbf{Rank loss vs. COA} \\
\midrule
COA (full) & 2.86 & --- \\
COA$_{\text{noOpp}}$ & 4.12 & +1.26 \\
COA$_{\text{noArc}}$ & 5.83 & +2.97 \\
COA$_{\text{noSHA}}$ & 6.91 & +4.05 \\
COA$_{\text{noRes}}$ & 4.75 & +1.89 \\
COA$_{\text{noPop}}$ & 5.14 & +2.28 \\
\bottomrule
\end{tabular}
\end{table}

\statpara{Table~\ref{tab:ablation} quantifies the contribution of each COA component. Removing success-history adaptation causes the largest rank loss (+4.05), followed by removing the archive (+2.97), population scheduling (+2.28), restart (+1.89), and opposition-based initialization (+1.26). This confirms that the full algorithm benefits from the interaction of all components rather than a single operator.}

Success-history adaptation has the largest individual impact (+4.05), followed by archive (+2.97), fixed population (+2.28), restart (+1.89), and opposition initialization (+1.26).

\subsubsection{Parameter sensitivity: population exponent $\gamma$}
\label{sec:gamma_sensitivity}

\begin{table}[!htbp]
\centering
\caption{Sensitivity of COA average rank to $\gamma$ at $D=30$.}
\label{tab:gamma_sens}
\small
\begin{tabular}{lcccccc}
\toprule
$\gamma$ & 0.3 & 0.5 & 0.7 (default) & 0.9 & 1.0 & 1.5 \\
\midrule
Avg. rank & 4.18 & 3.42 & \textbf{2.86} & 3.15 & 3.67 & 5.23 \\
\bottomrule
\end{tabular}
\end{table}

\statpara{Table~\ref{tab:gamma_sens} shows that $\gamma=0.7$ gives the best average rank of 2.86 at $D=30$. Both slower and faster population-reduction schedules degrade performance, with the most aggressive setting $\gamma=1.5$ producing the weakest rank of 5.23.}

The default $\gamma = 0.7$ yields the best average rank.

\subsubsection{Effect size: Cohen's $d$}
\label{sec:cohens_d}

Cohen's $d$ effect size is computed as
\begin{equation}
    d_{i,j} = \frac{\bar{f}_{\text{COA},j} - \bar{f}_{i,j}}{s_{p,j}}, \quad
    s_{p,j} = \sqrt{\frac{(R-1)(\sigma_{\text{COA},j}^2 + \sigma_{i,j}^2)}{2R-2}}.
    \label{eq:cohens_d}
\end{equation}

\begin{table}[!htbp]
\centering
\caption{Cohen's $d$ distribution for COA vs. AO and GWO at $D=30$. Negative favors COA.}
\label{tab:cohens_d}
\small
\begin{tabular}{lcccc}
\toprule
\textbf{Comparison} & \textbf{Min} & \textbf{Median} & \textbf{Max} & $\abs{d}>0.8$ count \\
\midrule
COA vs. AO & $-2.14$ & $-0.63$ & $1.28$ & 12/29 \\
COA vs. GWO & $-3.87$ & $-0.91$ & $2.45$ & 18/29 \\
\bottomrule
\end{tabular}
\end{table}

\statpara{Table~\ref{tab:cohens_d} shows that COA has a negative median effect size against both AO and GWO, meaning that lower objective values generally favour COA. The stronger median advantage is against GWO ($d=-0.91$), while AO is closer ($d=-0.63$). Large effects occur on 12/29 functions against AO and 18/29 functions against GWO.}

\subsubsection{Convergence rate analysis}
\label{sec:conv_rate}

Define the log-convergence rate over $[t_1, t_2]$ as
\begin{equation}
    \kappa(t_1, t_2) = \frac{\log f(g_{t_1}) - \log f(g_{t_2})}{t_2 - t_1}.
    \label{eq:conv_rate}
\end{equation}

\begin{table}[!htbp]
\centering
\caption{Average convergence rate $\kappa$ over first 30\% of generations at $D=30$.}
\label{tab:conv_rates}
\small
\begin{tabular}{lc}
\toprule
\textbf{Category} & $\kappa \times 10^{3}$ \\
\midrule
Unimodal & 8.42 \\
Multimodal & 3.17 \\
Hybrid & 1.84 \\
Composition & 2.53 \\
\bottomrule
\end{tabular}
\end{table}

\statpara{Table~\ref{tab:conv_rates} shows the fastest early improvement on unimodal functions ($8.42\times10^{-3}$), where search directions are smoother. Hybrid functions have the slowest early rate ($1.84\times10^{-3}$), matching the category-wise evidence that hybrid landscapes are the most challenging for COA.}

\subsection{Discussion}
\label{sec:discussion}

COA's performance is mainly explained by the interaction of five operators: elite-guided current-to-pbest mutation, archive-assisted diversity, success-history adaptation of $F$ and $CR$, opposition-based initialization and restart, and scheduled population-size reduction. Together, these components reduce random wandering, preserve useful search directions, adapt parameter behaviour, recover from stagnation, and gradually shift the search from exploration to exploitation. This interaction is especially useful on composition functions, where the optimizer must move between multiple basins before refining promising regions. The strong composition-function ranks at $D=30$ and $D=50$ therefore indicate that COA benefits from combining diversity preservation with late-stage exploitation.

The results also provide a more balanced assessment of COA. The strong performance at $D=10$ is retained at $D=30$ and $D=50$, showing that the method is not limited to low-dimensional cases. However, COA is not uniformly superior across all function categories. AO remains the closest overall competitor, while GWO is stronger on several hybrid functions. Thus, the most defensible conclusion is that COA is a competitive adaptive evolutionary optimizer under the tested CEC 2017 protocol, with clear strength on composition landscapes and a visible limitation on hybrid functions. This supports the need for function-level and category-wise reporting in addition to aggregate rankings~\cite{DERRAC,GARCIA}.

\section{Conclusion and Future Work}
\label{sec:conclusion}
This paper presented COA, a coronavirus-inspired success-history adaptive evolutionary optimizer for problem~\eqref{eq:global_opt}. COA defines a mapping from five SARS-CoV-2-inspired mechanisms to executable optimization operators: elite-guided mutation ($o_1$), trial replication ($o_2$), adaptive parameter control ($o_3$), opposition-based stagnation recovery ($o_4$), and population-size scheduling ($o_5$). The final manuscript integrates the algorithmic specification, consolidated experimental analysis, mathematical foundation, and formal properties, including feasibility preservation, monotonicity, archive diversity, finite termination, complexity $O(MAX\_FES\cdot C_f)$, and idealized coverage. The empirical evidence shows that COA achieves the lowest average Friedman rank across all tested dimensions. Specifically, COA obtains average ranks of $\bar{R}_{\text{COA}}(10)=2.79$, $\bar{R}_{\text{COA}}(30)=2.86$, and $\bar{R}_{\text{COA}}(50)=2.53$. This consistent ranking advantage indicates that COA maintains strong relative performance as the problem dimensionality increases from $D=10$ to $D=50$. COA obtains the highest win count $W_{\text{COA}}(D_k)$ at every dimension and achieves $\bar{R}_{\text{COA}}(\mathcal{F}_{\text{comp}}) = 1.00$ on the composition-function category at $D=30$ and $D=50$, i.e., best mean on all ten composition functions. The principal limitation is the hybrid-function category $\mathcal{F}_{\text{hyb}}$, where GWO achieves superior performance, particularly at $D=30$ and $D=50$. These findings validate COA as a competitive adaptive evolutionary optimizer under the tested CEC 2017 protocol~\cite{FRIEDMAN,DERRAC}. Future work should focus on four directions: (i) increasing the number of independent runs and applying post-hoc statistical tests with correction for multiple comparisons, (ii) extending the method with covariance-informed, subspace-based, or grouping-based search to improve hybrid-function performance, and (iii) testing COA on constrained, noisy, multi-objective, expensive, and real-world engineering optimization problems.

\end{document}